%% file: main.tex
\documentclass{article}

\usepackage{microtype}
\usepackage{graphicx}
\usepackage{booktabs} 

\usepackage{verbatim}
\usepackage{enumitem}
\usepackage{algorithmic}
\usepackage{amsmath,amssymb,amsthm}
\usepackage{mathrsfs}
\usepackage{bbm}
\usepackage{xcolor}
\usepackage{thm-restate}
\usepackage{subfig}

\usepackage{wrapfig}

\DeclareMathOperator*{\argmin}{\arg\!\min}

\usepackage{hyperref}

\usepackage[accepted]{icml2019}

\newtheorem{theorem}{Theorem}[section]
\newtheorem{lemma}[theorem]{Lemma}

\newtheorem{example}[theorem]{Example}
\newtheorem{definition}[theorem]{Definition}

\DeclareMathOperator*{\argmax}{arg\,max}

\newcommand{\statespace}{\mathcal{X}}
\newcommand{\actionspace}{\mathcal{A}}
\newcommand{\state}{x}
\newcommand{\action}{a}
\newcommand{\reward}{r}
\newcommand{\statevar}{X}
\newcommand{\actionvar}{A}
\newcommand{\discountrate}{\gamma}
\newcommand{\rewarddist}[2]{\mathcal{R}(#1, #2)}
\newcommand{\statistic}{s}
\newcommand{\gendist}{\mu}
\newcommand{\genrv}{Z}

\newcommand{\policy}{\pi}
\newcommand{\expectileval}{\epsilon}
\newcommand{\expectilestat}{e}
\newcommand{\numstats}{K}
\newcommand{\statix}{k}
\newcommand{\numsamples}{N}
\newcommand{\sample}{z}
\newcommand{\sampleix}{n}

\newcommand{\rewardvar}{R}

\newcommand{\returndist}[1]{\eta_{#1}}
\newcommand{\BellmanOp}[2]{\mathcal{T}^{#1}_{#2}}
\newcommand{\statval}{\sigma}

\newcommand{\cdrlDepthIx}{L}
\newcommand{\lemmaindex}{m}
\newcommand{\approxstatistic}{\hat{\statistic}}
\newcommand{\approxdist}{\eta}

\newcommand{\expected}{\mathbb{E}}
\newcommand{\supportbound}{R_{\text{max}}}
\newcommand{\support}{\text{supp}}
\newcommand{\supportwidth}{I}

\icmltitlerunning{Statistics and Samples in Distributional Reinforcement Learning}

\begin{document}

\twocolumn[
\icmltitle{Statistics and Samples in Distributional Reinforcement Learning}




\begin{icmlauthorlist}
\icmlauthor{Mark Rowland}{dm}
\icmlauthor{Robert Dadashi}{gb}
\icmlauthor{Saurabh Kumar}{gb}
\icmlauthor{R\'emi Munos}{dm}
\icmlauthor{Marc G. Bellemare}{gb}
\icmlauthor{Will Dabney}{dm}
\end{icmlauthorlist}

\icmlaffiliation{dm}{DeepMind}
\icmlaffiliation{gb}{Google Brain}

\icmlcorrespondingauthor{Mark Rowland}{markrowland@google.com}

\icmlkeywords{Machine Learning, ICML, Reinforcement Learning, Distributional, Distributional Reinforcement Learning}

\vskip 0.3in
]



\printAffiliationsAndNotice{\icmlEqualContribution} 

\begin{abstract}
We present a unifying framework for designing and analysing distributional reinforcement learning (DRL) algorithms in terms of recursively estimating statistics of the return distribution. Our key insight is that DRL algorithms can be decomposed as the combination of some statistical estimator and a method for imputing a return distribution consistent with that set of statistics. With this new understanding, we are able to provide improved analyses of existing DRL algorithms as well as construct a new algorithm (EDRL) based upon estimation of the \textit{expectiles} of the return distribution. We compare EDRL with existing methods on a variety of MDPs to illustrate concrete aspects of our analysis, and develop a deep RL variant of the algorithm, ER-DQN, which we evaluate on the Atari-57 suite of games.
\end{abstract}

\input{introduction}

\input{drl}
\input{statsAndSamples}
\input{analysis}
\input{experiments}
\input{conclusion}

\section*{Acknowledgements}
The authors acknowledge the vital contributions of their colleagues at DeepMind. Thanks to Hado van Hasselt for detailed comments on an earlier draft, and to Georg Ostrovski for useful suggestions regarding the SciPy optimisation calls within ER-DQN.

\newpage

\bibliography{erdqn}
\bibliographystyle{icml2019}

\newpage
\onecolumn

\appendix
\addcontentsline{toc}{section}{Appendices}
\section*{Appendices}

\input{appendixAlgorithms}
\input{appendixProofs}    
\input{appendixAdditionalResults}
\input{appendixExperiments}
\newpage
\input{appendixExamples}

\end{document}

%% file: introduction.tex
\section{Introduction}\label{sec:intro}

In reinforcement learning (RL), a central notion is the \textit{return}, the sum of discounted rewards. Typically, the average of these returns is estimated by a value function and used for policy improvement. Recently, however, approaches that attempt to learn the distribution of the return have been shown to be surprisingly effective \citep{MorimuraNonparametric,MorimuraParametric,C51,QRDQN,IQN,gruslys2018the}; we refer to the general approach of learning return distributions as \emph{distributional RL} (DRL).

Despite impressive experimental performance \citep{C51,D4PG,IQN} and fundamental theoretical results \citep{AnalysisCDRL,qu2018nonlinear}, it remains challenging to develop and analyse DRL algorithms. In this paper, we propose to address these challenges by phrasing DRL algorithms in terms of recursive estimation of sets of statistics on the return distribution. We observe that DRL algorithms can be viewed as combining a statistical estimator with a procedure we refer to as an \textit{imputation strategy}, which generates a return distribution consistent with the set of statistical estimates. This highly general approach (see Figure~\ref{fig:idrl_intro}) requires a precise treatment of the differing roles of statistics and samples in distributional RL.

Using this framework we are able to provide new theoretical results for existing DRL algorithms as well as demonstrate the derivation of a new algorithm based on the expectiles of the return distribution. More importantly, our novel approach immediately applies to a large class of statistics and imputation strategies, suggesting several avenues for future research.
Specifically, we are able to provide answers to the following questions:

\begin{enumerate}[label=(\roman*),leftmargin=0.6cm,topsep=-1pt,itemsep=-1ex,partopsep=1ex,parsep=1ex]
    \item Can we describe existing DRL algorithms in a unifying framework, and could such a framework be used to develop new algorithms?
    \item What return distribution statistics can be learnt \emph{exactly} through Bellman updates?
    \item If certain statistics cannot be learnt exactly, how can we estimate them in a principled manner, and give guarantees on their approximation error relative to the true values of these statistics?
\end{enumerate}

After reviewing relevant background material, we begin with (i) by presenting a new framework for understanding DRL, that is, in terms of a set of \emph{statistics} to be learnt, and an \emph{imputation strategy} for specifying a dynamic programming update.
We then
formalise (ii) by
introducing the notion of \emph{Bellman closedness} for collections of statistics, and show that in a wide class of statistics, the only properties of return distributions that can be learnt exactly through Bellman updates are moments. Interestingly, this rules out statistics such as quantiles that have formed the basis of successful existing DRL algorithms.
However, we then
address (iii) by
showing that the framework allows us to give guarantees on the approximation error introduced in learning these statistics, through the notion of \emph{approximate Bellman closedness}. 
We apply the framework developed in answering these questions to the case of \emph{expectile} statistics to develop a new distributional RL algorithm, which we term Expectile Distributional RL (EDRL).  
Finally, we test these new insights on a variety of MDPs
and larger-scale environments
to illustrate and expand on the theoretical contributions developed earlier in the paper.

\begin{figure}
    \centering
    \includegraphics[keepaspectratio,width=0.48\textwidth]{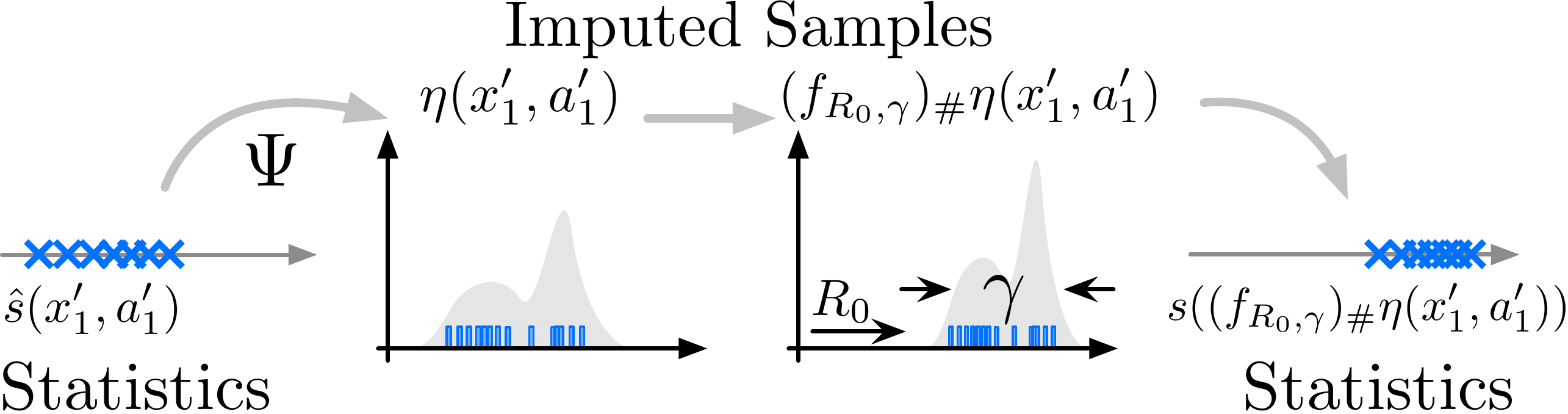}
    \caption{Illustration of learning with imputed samples from sets of statistics. Left: A distribution is imputed from the current statistical estimate. Middle: The distributional Bellman operator is applied to the imputed distribution. Right: New statistics are estimated based upon samples from the imputed distribution.}
    \label{fig:idrl_intro}
\end{figure}

%% file: drl.tex
\section{Background}\label{sec:background}
Consider a Markov decision process $(\statespace, \actionspace, p, \gamma, \mathcal{R})$ with finite state space $\statespace$, finite action space $\actionspace$, transition kernel $p: \statespace \times \actionspace \rightarrow \mathscr{P}(\statespace)$, discount rate $\discountrate \in [0, 1)$, and reward distributions $\rewarddist{\state}{\action} \in \mathscr{P}(\mathbb{R})$ for each $(\state, \action) \in \statespace \times \actionspace$. Thus, if an agent is at state $\statevar_t \in \statespace$ at time $t \in \mathbb{N}_0$, and an action $\actionvar_t \in \actionspace$ is taken, the agent transitions to a state $\statevar_{t+1} \sim p(\cdot| \statevar_{t}, \actionvar_{t})$ and receives a reward $\rewardvar_t \sim \rewarddist{\statevar_t}{\actionvar_t}$.
We now briefly review two principal goals in reinforcement learning.

Firstly, given a Markov policy $\policy : \statespace \rightarrow \mathscr{P}(\actionspace)$, \emph{evaluation} of $\policy$ consists of computing the expected returns $Q^\policy(\state, \action) = \mathbb{E}_\policy\left\lbrack \sum_{t=0}^\infty \gamma^t \rewardvar_t | \statevar_0 = \state, \actionvar_0 = \action \right\rbrack$, where $\mathbb{E}_\pi$ indicates that at each time step $t \in \mathbb{N}$, the agent's action $A_t$ is sampled from $\pi(\cdot | \statevar_t)$. Secondly, the task of \emph{control} consists of finding a policy $\policy : \statespace \rightarrow \mathscr{P}(\actionspace)$ for which the expected returns are maximised.

\subsection{Bellman equations}

The classical \emph{Bellman equation} \citep{Bellman} relates expected returns at each state-action pair $(x, a) \in \statespace \times \actionspace$ to the expected returns at possible next states in the MDP by:
\begin{align}\label{eq:meanBellman}
    Q^\pi\!(\state, \action)\! =\! \mathbb{E}_{\policy}\!\left\lbrack \rewardvar_0 \!+\! \gamma Q^\policy(\statevar_1, \actionvar_1 ) | \statevar_0 \!=\! \state, \actionvar_0 \!=\! \action \right\rbrack .
\end{align}
This gives rise to the following fixed-point iteration scheme
\begin{align}\label{eq:meanBackup}
    Q(\state, \action)\! \leftarrow\! \mathbb{E}_{\policy}\!\left\lbrack \rewardvar_0 + \gamma Q(\statevar_1, \actionvar_1 ) | \statevar_0 = \state, \actionvar_0 = \action \right\rbrack ,
\end{align}
for updating a collection of \emph{approximations} $(Q(\state, \action) | (\state, \action) \in \statespace \times \actionspace)$ towards their true values.
This fundamental algorithm, together with techniques from approximate dynamic programming and stochastic approximation, allows expected returns in an MDP to be learnt and improved upon, forming the basis of all value-based RL \citep{sutton2018reinforcement}.

The \emph{distributional Bellman equation} describes a similar relationship to Equation \eqref{eq:meanBellman} at the level of probability distributions \citep{MorimuraNonparametric,MorimuraParametric,C51}. Letting $\returndist{\policy}(\state, \action) \in \mathscr{P}(\mathbb{R})$ be the \emph{distribution} of the random return $\sum_{t=0}^\infty \gamma^t \rewardvar_t \ | \statevar_0 = \state, \actionvar_0 = \action$ when actions are selected according to $\policy$, we have
\begin{align}\label{eq:distBellmanEq}
    \returndist{\policy}(\state, \action) &= (\mathcal{T}^\pi \returndist{\policy})(\state, \action),\\
   \nonumber &= \mathbb{E}_\pi \!\left\lbrack (f_{\rewardvar_0, \gamma})_\# \returndist{\policy}(\statevar_1, \actionvar_1) |  \statevar_0 \!=\! \state, \actionvar_0 \!=\! \action \right\rbrack ,
\end{align}
where the expectation gives a mixture distribution over next-states, $f_{r, \gamma} : \mathbb{R} \rightarrow \mathbb{R}$ is defined by $f_{r, \gamma}(x) = r + \gamma x$, and $g_\# \gendist \in \mathscr{P}(\mathbb{R})$ is the pushforward of the measure $\mu$ through the function $g$, so that for all Borel subsets $A \subseteq \mathbb{R}$, we have $g_\# \mu (A) = \mu(g^{-1}(A))$ \citep{AnalysisCDRL}.

Stated in terms of the random return $Z^\pi(\state, \action)$, distributed according to $\returndist{\policy}(\state, \action)$, this takes a more familiar form with

\begin{equation*}
    Z^\pi(\state, \action) \overset{D}{=} \rewardvar_0 + \gamma Z^\pi(\statevar_1, \actionvar_1) \, .
\end{equation*}

In analogy with Expression \eqref{eq:meanBackup}, an update operation could be defined from Equation \eqref{eq:distBellmanEq} to move a collection of approximate distributions $(\approxdist(\state, \action) | (\state, \action) \in \statespace \times \actionspace)$ towards the true return distributions.
However, since the space of distributions $\mathscr{P}(\mathbb{R})$ is infinite-dimensional, it is typically impossible to work directly with the distributional Bellman equation, and existing approaches to distributional RL generally rely on parametric approximations to this equation; we briefly review some important examples of these approaches below.

\subsection{Categorical and quantile distributional RL}\label{sec:CDRLQDRLDescriptions}

To date, the main approaches to DRL employed at scale have included learning discrete \emph{categorical} distributions \citep{C51,D4PG,qu2018nonlinear}, and learning distribution \emph{quantiles} \citep{QRDQN,IQN,zhang2018quota}; we refer to these approaches as CDRL and QDRL respectively. We give brief accounts of the dynamic programming versions of these algorithms here, with full descriptions of stochastic versions, related results, and visualisations given in Appendix Section~\ref{sec:olderalgorithms} for completeness. We note also that other approaches, such as learning mixtures of Gaussians, have been explored \citep{D4PG}.

\textbf{CDRL.} CDRL assumes a \emph{categorical} form for return distributions, taking $\approxdist(\state, \action) = \sum_{\statix=1}^\numstats p_\statix(\state, \action) \delta_{z_\statix}$, where $\delta_{z}$ denotes the Dirac distribution at location $z$. The values $z_1 < \cdots < z_\numstats$ are an evenly spaced, fixed set of supports, and the probability parameters $p_{1:\numstats}(\state, \action)$ are learnt. The corresponding Bellman update takes the form
\begin{align*}
    \approxdist(\state, \action) \leftarrow (\Pi_\mathcal{C} \mathcal{T}^\pi \approxdist) (\state, \action),
\end{align*}
where $\Pi_\mathcal{C} : \mathscr{P}(\mathbb{R}) \rightarrow \mathscr{P}(\{z_1,\ldots,z_\numstats\})$ is a projection operator which ensures the right-hand side of the expression above is a distribution supported only on $\{z_1,\ldots,z_\numstats\}$; full details are reviewed in Appendix Section~\ref{sec:olderalgorithms}.

\textbf{QDRL.} In contrast, QDRL assumes a parametric form for return distributions $\approxdist(\state, \action) = \frac{1}{\numstats} \sum_{\statix=1}^\numstats \delta_{z_\statix(\state, \action)}$, where now $z_{1:\numstats}(\state, \action)$ are learnable parameters. The Bellman update is given by moving the atom location $z_\statix(\state, \action)$ in $\approxdist(\state, \action)$ to the $\tau_k$-quantile (where $\tau_k =\frac{2\statix - 1}{2\numstats}$) of the target distribution $\mu := (\mathcal{T}^\pi \approxdist)(\state, \action)$, defined as the minimiser $q^* \in \mathbb{R}$ of the quantile regression loss
\begin{align}\label{eq:qr}
    \mathrm{QR}(q; \gendist, \tau_k) =
    \mathbb{E}_{\genrv \sim \gendist}\!\left\lbrack \left\lbrack \tau_\statix \mathbbm{1}_{\genrv > q} + (1 - \tau_\statix) \mathbbm{1}_{\genrv \leq q} \right\rbrack |\genrv - q| \right\rbrack  .
\end{align}

%% file: statsAndSamples.tex
\section{The role of statistics in distributional RL}\label{sec:statsAndSamples}

In this section, we describe a new perspective on existing distributional RL algorithms, with a focus on learning sets of statistics, rather than approximate distributions. We begin with a precise definition.

\begin{definition}[\textbf{Statistics}]
    A \emph{statistic} is a function $s : \mathscr{P}(\mathbb{R}) \rightarrow \mathbb{R}$. We also allow statistics to be defined on subsets of $\mathscr{P}(\mathbb{R})$, in situations where an assumption (such as finite moments) is required for the statistic to be defined.
\end{definition}

The QDRL update described in Section~\ref{sec:CDRLQDRLDescriptions} is readily interpreted from the perspective of learning statistics; the update extracts the values of a finite set of quantile statistics from the target distribution, and all other information about the target is lost. It is less obvious whether the CDRL update can also be interpreted as keeping track of a finite set of statistics, but the following lemma shows that this is indeed the case.

\begin{restatable}{lemma}{CDRLStats}\label{lem:CDRLStats}
    CDRL updates, with distributions supported on $z_1 < \ldots < z_K$, can be interpreted as learning the values of the following statistics of return distributions:
    \begin{align*}
        s_{z_{\statix}, z_{\statix+1}}(\gendist) \!=\! \mathbb{E}_{\genrv \sim \gendist}\!\left\lbrack h_{z_{\statix}, z_{\statix+1}}(\genrv) \right\rbrack \textrm{\ for\ \ } \statix\!=\!1,\ldots,\numstats\!-\!1 \, ,
    \end{align*}
    where for $a<b$, $h_{a, b} : \mathbb{R} \rightarrow \mathbb{R}$ is a piecewise linear function defined so that $h_{a,b}(x)$ is equal to $1$ for $x \leq a$, equal to $0$ for $x \geq b$, and linearly interpolating between $h_{a,b}(a)$ and $h_{a,b}(b)$ for $x \in [a,b]$.
\end{restatable}

Although viewing distributional RL as approximating the return distribution with some parameterisation is intuitive from an algorithmic standpoint, there are advantages to thinking in terms of sets of statistics and their recursive estimation; this perspective allows us to precisely quantify what information is being passed through successive distributional Bellman updates. This in turn leads to new insights in the development and analysis of DRL algorithms. Before addressing these points, we first consider a motivating example where a lack of precision could lead us astray.

\subsection{Expectiles}\label{sec:expectilemotivation}

Motivated by the success of QDRL, we consider learning \emph{expectiles} of return distributions, a family of statistics introduced by \citet{newey1987asymmetric}. Expectiles generalise the mean in analogy with how quantiles generalise the median. As the goal of RL is to maximise mean returns, we conjectured that expectiles, in particular, might lead to successful DRL algorithms. We begin with a formal definition.
\begin{definition}[\textbf{Expectiles}]\label{def:expectiles}
Given a distribution $\gendist \in \mathscr{P}(\mathbb{R})$ with finite second moment, and $\tau \in [0,1]$, the $\tau$-expectile of $\gendist$ is defined to be the minimiser $q^* \in \mathbb{R}$ of the expectile regression loss $\mathrm{ER}(q; \gendist, \tau)$, given by
\begin{align*}
    \mathrm{ER}(q; \gendist, \tau) = \mathbb{E}_{\genrv \sim \gendist}\left\lbrack \left\lbrack \tau \mathbbm{1}_{\genrv > q} + (1 - \tau) \mathbbm{1}_{\genrv \leq q} \right\rbrack (\genrv - q)^2 \right\rbrack .
\end{align*}
For each $\tau \in [0,1]$, we denote the $\tau$-expectile of $\gendist$ by $\expectilestat_\tau(\gendist)$.
\end{definition}
We remark that: (i) the expectile regression loss is an asymmetric version of the squared loss, just as the quantile regression loss is an asymmetric version of the absolute value loss; and (ii) the $1/2$-expectile of $\gendist$ is simply its mean. Because of this, we can attempt to derive an algorithm by replacing the quantile regression loss in QDRL with the expectile regression loss in Definition~\ref{def:expectiles}, so as to learn the expectiles corresponding to $\tau_1,\ldots,\tau_\numstats \in [0,1]$.

Following this logic, we again take approximate distributions of the form $\approxdist(\state, \action) = \frac{1}{\numstats} \sum_{\statix=1}^\numstats \delta_{z_\statix(\state, \action)}$, and we perform updates according to
\begin{align}\label{eq:naiveExpectileUpdate}
    z_\statix(\state, \action) \leftarrow \argmin_{q \in \mathbb{R}} \mathrm{ER}(q; \mu, \tau_k) \, ,
\end{align}
where $\mu = (\mathcal{T}^\pi \returndist{})(\state, \action)$ is the target distribution.

\begin{figure}[t]
    \centering
    \includegraphics[keepaspectratio, width=.45\textwidth]{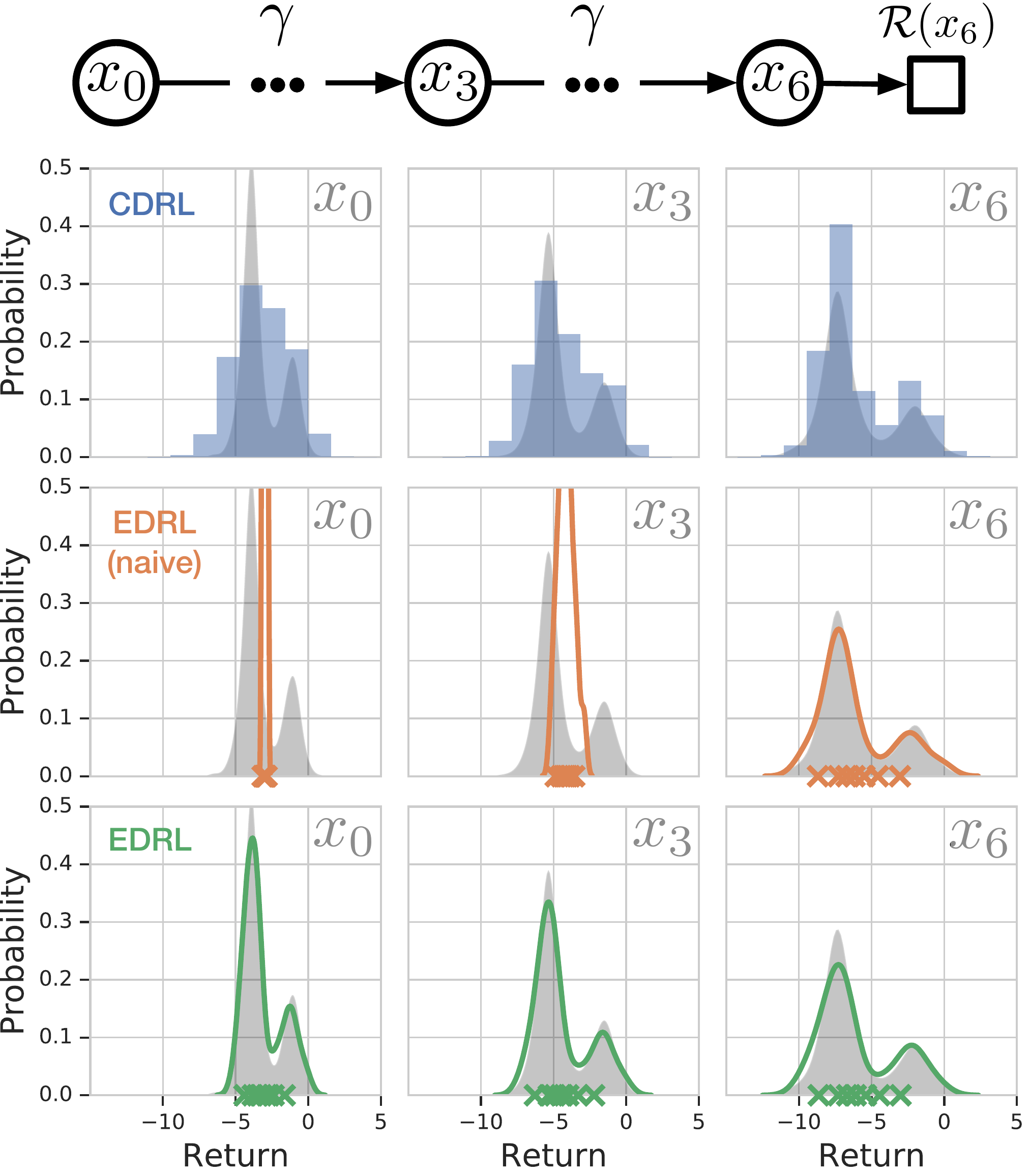}
    \caption{Chain MDP, one action, with bimodal reward distribution at absorbing state $x_6$ and $\gamma = 0.9$. CDRL (top, blue) fits the true return distribution (grey) well, but overestimates the variance. A naive approach to EDRL (middle, orange) accurately fits the immediate reward distribution at $x_6$, but quickly collapses to zero variance with successive Bellman updates. Our proposed approach, EDRL, using imputation strategies (bottom, green) provides an accurate approximation through many Bellman updates.}
    \label{fig:chainExample}
\end{figure}

In practice, however, this algorithm does not perform as we might expect, and in fact the variance of the learnt distributions collapses as training proceeds, indicating that the algorithm does not approximate the true expectiles in any reasonable sense. In Figure~\ref{fig:chainExample}, we illustrate this point by comparing the learnt statistics for this ``naive'' approach with those of CDRL and our proposed algorithm EDRL (introduced in Section~\ref{sec:expectiles}). All methods accurately approximate the immediate reward distribution (right), but as successive Bellman updates are applied the different algorithms show characteristic approximation errors. The CDRL algorithm overestimates the variance of the return distribution due to the projection $\Pi_\mathcal{C}$ splitting probability mass across the discrete support. By contrast, the naive expectile approach underestimates the true variance, quickly converging to a single Dirac.

We observe that there is a \emph{``type error''} present in Expression~\eqref{eq:naiveExpectileUpdate}; the parameter being updated, $z_{\statix}(\state, \action)$, has the semantics of a \emph{statistic}, as the minimiser of the $\mathrm{ER}$ loss, whilst the parameters appearing in the target distribution $(\mathcal{T}^\pi \approxdist)(\state, \action)$ have the semantics of \emph{outcomes}/\emph{samples}.
A crucial message of this paper is the need to distinguish between statistics and samples in distributional RL; in the next section, we describe a general framework for achieving this.

\subsection{Imputation strategies}\label{sec:imputationstrategies}

If we had access to full return distribution estimates $\approxdist(\state^\prime, \action^\prime)$ at each possible next state-action pair $(\state^\prime, \action^\prime)$, we would be able to avoid the conflation between samples and statistics described in the previous section. Denoting the approximation to the value of a statistic $\statistic_\statix$ at a state-action pair $(\state, \action) \in \statespace\times\actionspace$ by $\approxstatistic_\statix(\state, \action)$, we would like to update according to:
\begin{align}\label{eq:imputedBackup}
    \approxstatistic_\statix(\state, \action) \!\leftarrow\! \statistic_\statix\left((\mathcal{T}^\pi \approxdist)(\state, \action)\right) .
\end{align}
Thus, a principled way in which to design DRL algorithms for collections of statistics 
is to include an \emph{additional} step in the algorithm in which for any state-action pair $(\state^\prime, \action^\prime)$ that we would like to backup from, the estimated statistics $\approxstatistic_{1:\numstats}(\state^\prime, \action^\prime)$ are converted into a consistent distribution $\approxdist(\state^\prime, \action^\prime)$. This would then allow backups of the form in Expression~\eqref{eq:imputedBackup} to be carried out. This notion is formalised in the following definition.

\begin{definition}[\textbf{Imputation strategies}]
    Given a set of statistics $\{\statistic_1, \ldots, \statistic_\numstats\}$, an \emph{imputation strategy} is a function $\Psi : \mathbb{R}^\numstats \rightarrow \mathscr{P}(\mathbb{R})$ that maps each vector of statistic values to a distribution that has those statistics. Mathematically, $\Psi$ is such that $\statistic_i(\Psi(\sigma_{1:\numstats})) = \sigma_i$, for each $i \in \{1,\ldots,\numstats\}$ and each collection of statistic values $\sigma_{1:\numstats} \in \mathbb{R}^\numstats$.
\end{definition}

Thus, an imputation strategy is simply a function that takes in a collection of values for certain statistics, and returns a probability distribution with those statistic values; in some sense, it is a pseudo-inverse of $\statistic_{1:\numstats}$.

\begin{example}[\textbf{Imputation strategies in CDRL and QDRL}]
In QDRL, the imputation strategy is given by $\Psi(\statval_{1:\numstats}) = \frac{1}{\numstats}\sum_{\statix=1}^\numstats \delta_{\statval_\statix}$. In CDRL, given approximate statistics $\approxstatistic_{z_\statix, z_{\statix+1}}(\state, \action)$ for $k=1,\ldots,\numstats-1$, the imputation strategy is given by selecting the distribution $\sum_{\statix=1}^\numstats p_{\statix} \delta_{z_\statix}$ such that $p_1 = \approxstatistic_{z_1, z_2}(\state, \action)$, $p_\statix = \approxstatistic_{z_{\statix}, z_{\statix+1}}(\state, \action) - \approxstatistic_{z_{\statix-1},z_{\statix}}(\state, \action)$ for $\statix=2,\ldots,\numstats-1$, and $p_\numstats = 1- \sum_{\statix < \numstats} p_\statix$.
\end{example}
We now have a general framework for defining principled distributional RL algorithms: (i) select a family of statistics to learn; (ii) select an imputation strategy; (iii) perform (or approximate) updates of the form in Expression~\eqref{eq:imputedBackup}. We summarise this in Algorithm \ref{alg:genDRL}.

 \begin{algorithm}
     \begin{algorithmic}
     \REQUIRE Statistic estimates $\approxstatistic_{1:\numstats}(\state, \action)$ $\forall (\state, \action) \in \statespace \times \actionspace$ and $\statix=1,\ldots,\numstats$, imputation strategy $\Psi$.
     \STATE Select state-action pair $(\state, \action) \in \statespace \times \actionspace$ to update.
     \STATE Impute distribution at each possible next state-action pair:
     \STATE $\quad \approxdist(\state^\prime, \action^\prime) = \Psi(\approxstatistic_{1:\numstats}(\state^\prime, \action^\prime)),
    \quad \forall (\state^\prime, \action^\prime) \in \statespace \times \actionspace$.
     \STATE Update statistics at $(\state, \action) \in \statespace\times\actionspace$:
     \STATE $\quad \approxstatistic_\statix(\state, \action) \!\leftarrow\! \statistic_\statix\left((\mathcal{T}^\pi \approxdist)(\state, \action) \right)$.
     \end{algorithmic}
     \caption{Generic DRL update algorithm.}
     \label{alg:genDRL}
 \end{algorithm}

\subsection{Expectile distributional reinforcement learning}\label{sec:expectiles}
We now apply the general framework of statistics and imputation strategies developed in Section~\ref{sec:imputationstrategies} to the specific case of \emph{expectiles}, introduced in Section~\ref{sec:expectilemotivation}.
We will define an imputation strategy so that updates of the form given in Expression~\eqref{eq:imputedBackup} can be applied to learn expectiles.

The imputation strategy has the task of accepting as input a  collection of expectile values $\expectileval_{1}, \ldots, \expectileval_{\numstats}$, corresponding to $\tau_1,\ldots,\tau_\numstats \in (0,1)$, and computing a probability distribution $\mu$ such that $\expectilestat_{\tau_i}(\mu) = \expectileval_{i}$ for $i = 1, \ldots, \numstats$. Since $\text{ER}(q; \gendist, \tau)$ is strictly convex as a function of $q$, this can be restated as finding a probability distribution $\gendist$ satisfying the first-order optimality conditions
\begin{align}\label{eq:rootFindingProblem}
    \nabla_q \text{ER}(q; \gendist, \tau_i ) \big|_{q = \expectileval_i} \!\! = 0 \ \ \ \forall i \in [\numstats] \, .
\end{align}
This defines a root-finding problem, but may equivalently be formulated as a minimisation problem, with objective
\begin{align}\label{eq:minProblem}
    \sum_{i=1}^\numstats \left( \nabla_q \text{ER}(q; \gendist, \tau_i ) \big|_{q = \expectileval_i} \right)^2 \, .
\end{align}
By constraining the distribution $\mu$ to be of the form $\frac{1}{\numsamples} \sum_{\sampleix=1}^\numsamples \delta_{\sample_\sampleix}$ and viewing the minimisation objective above as a function of $z_{1:\numsamples}$, it is straightforwardly verifiable that this minimisation problem is convex. The imputation strategy is thus defined implicitly, by stating that $\Psi(\epsilon_{1:\numstats})$ is given by a minimiser of \eqref{eq:minProblem} of the form $\frac{1}{\numsamples} \sum_{\sampleix=1}^\numsamples \delta_{\sample_\sampleix}$. We remark that other parametric choices for $\gendist$ are possible, but the mixture of Dirac deltas described above leads to a particular tractable optimisation problem.

Having established an imputation strategy $\Psi$, Algorithm~\ref{alg:genDRL} now yields a full DRL algorithm for learning expectiles, which we term EDRL. Returning to Figure~\ref{fig:chainExample}, we observe that EDRL (bottom row) is able to accurately represent the true return distribution, even after many Bellman updates through the chain, and does not exhibit the collapse observed with the naive approach in Section~\ref{sec:expectilemotivation}.

\subsection{Stochastic approximation}\label{sec:SA}
Practically speaking, it is often not possible to compute the updates in Expression~\eqref{eq:imputedBackup}, owing to MDP dynamics being unknown and/or intractable to integrate over.
Because of this, it is often necessary to apply stochastic approximation. Let $(\reward, \state^\prime, \action^\prime)$ be a sample of the random variables $(\rewardvar_0, \statevar_1, \actionvar_1)$, obtained by direct interaction with the environment. Then, we update $\approxstatistic_\statix(\state, \action)$ using the gradient of a loss function $L_\statix : \mathbb{R} \times \mathscr{P}(\mathbb{R}) \rightarrow \mathbb{R}$:
\begin{align}\label{eq:grad}
    \nabla_{\approxstatistic_\statix(\state, \action)} L_\statix(\approxstatistic_\statix(\state, \action); (f_{\reward, \gamma})_\# \approxdist(\state^\prime, \action^\prime)) \, .
\end{align}
For EDRL, a natural such loss function for the estimated statistic $\approxstatistic_\statix(\state, \action)$ is the expectile regression loss of Definition \ref{def:expectiles} at $\tau_\statix$; this yields a stochastic version of EDRL, described in Algorithm \ref{alg:EDRL}.

\begin{algorithm}
    \begin{algorithmic}
    \REQUIRE Expectile estimates $\approxstatistic_\statix(\state, \action)$ for each $(\state, \action) \in \statespace \times \actionspace$ and $\statix=1,\ldots,\numstats$.
    \STATE Collect sample $(\state, \action, \reward, \state^\prime, \action^\prime)$.
    \STATE Impute distribution $\frac{1}{\numstats} \sum_{\statix=1}^\numstats \delta_{z_\statix}$ from target expectiles $\approxstatistic_{1:\numstats}(\state^\prime, \action^\prime)$ by solving \eqref{eq:rootFindingProblem} or minimising \eqref{eq:minProblem}.
    \STATE Scale/translate samples $z_i \leftarrow r + \gamma z_i$ $\forall i$.
    \STATE Update estimated expectiles at $(\state, \action) \in \statespace \times \actionspace$ by computing the gradients
    \begin{align*}
        \textstyle \nabla_{\approxstatistic_\statix(\state, \action)} \sum_{\statix=1}^\numstats \text{ER}(\approxstatistic_\statix(\state, \action); \frac{1}{\numsamples}\sum_{\sampleix=1}^\numsamples \delta_{z_\sampleix}, \tau_\statix)
    \end{align*}
    for each $\statix=1,\ldots,\numstats$.
    \end{algorithmic}
    \caption{Stochastic EDRL update algorithm.}
    \label{alg:EDRL}
\end{algorithm}

To ensure convergence of these stochastic gradient updates to the correct statistic, it should be the case that the expectation of the (sub-)gradient \eqref{eq:grad} at the true value of the statistics is equal to $0$. It can be verified that this is the case whenever (i) the true statistic $q^*$ of a distribution $\mu$ satisfies $q^* = \argmin_{q \in \mathbb{R}} L_k(q; \mu)$, (ii) the loss $L_k$ is \emph{affine} in the probability distribution argument. M-estimator losses and their associated statistics \citep{RobustStatistics} satisfy these conditions, and thus represent a large family of statistics to which this approach to DRL could immediately be applied; the statistics in CDRL, QDRL and EDRL are all special cases of M-estimators.

%% file: analysis.tex
\section{Analysing distributional RL}\label{sec:analysis}

We now use the framework of statistics and imputations strategies developed in Section \ref{sec:statsAndSamples} to build a deeper understanding of the accuracy with which statistics in distributional RL may be learnt via Bellman updates.

\subsection{Bellman closedness}\label{subsec:bellmanclosed}

The classical Bellman equation \eqref{eq:meanBellman} shows that there is a closed-form relationship between expected returns at each state-action pair of an MDP; if the goal is to learn expected returns, we are not required to keep track of any other statistics of the return distributions.
This well-known observation, together with the new interpretation of DRL algorithms as learning collections of statistics of return distributions, motivates a more general question:

\begin{center}
    \begin{minipage}{.45\textwidth}
        ``Given a set of statistics $\{\statistic_1,\ldots,\statistic_\numstats\}$, if we want to learn the values $\statistic_{1:\numstats}(\returndist{\policy}(x, a))$ for all $(\state, \action) \in \statespace \times \actionspace$ via dynamic programming, is it sufficient to keep track of \emph{only} these statistics?''
    \end{minipage}
\end{center}

The following definition formalises this question.

\begin{definition}[\textbf{Bellman closedness}]\label{def:bellmanclosed}
     A set of statistics $\{\statistic_1, \ldots, \statistic_\numstats\}$ is \emph{Bellman closed} if for each $(\state , \action) \in \statespace \times \actionspace$, the statistics $\statistic_{1:\numstats}(\returndist{\policy}(\state, \action))$ can be expressed, in an MDP-independent manner, in terms of the random variables $\rewardvar_0$ and $\statistic_{1:\numstats}(\returndist{\policy}(\statevar_1, \actionvar_1)) | \statevar_0 = \state, \actionvar_0 = \action$, and the discount factor $\gamma$. We refer to any such expression for a set of Bellman closed set of statistics as a \emph{Bellman equation}, and write $\BellmanOp{\policy}{{}} : (\mathbb{R}^{\numstats})^{\statespace \times \actionspace} \rightarrow (\mathbb{R}^{\numstats})^{\statespace \times \actionspace}$ for the corresponding operator such that the Bellman equation can be written
     \begin{align}
         \mathbf{s}^\policy = \BellmanOp{\policy}{{}} \mathbf{s}^\policy \, ,
     \end{align}
     where $\mathbf{s}^\policy = (\statistic_{1:\numstats}(\returndist{\policy}(\state, \action)) | (\state, \action) \in \statespace \times \actionspace)$.
\end{definition}

Thus, the singleton set consisting of the mean statistic is Bellman closed; the corresponding Bellman equation is Equation~\eqref{eq:meanBellman}. It is also known that the set consisting of the mean and variance statistics are Bellman closed \citep{sobel1982variance}.
In principle, given a Bellman closed set of statistics $\{\statistic_1, \ldots, \statistic_\numstats\}$, the corresponding statistics of the return distributions can be found by solving a fixed-point equation corresponding to the relevant Bellman operator, $\BellmanOp{\policy}{}$. Further, if $\BellmanOp{\policy}{}$ is a contraction in some metric, then it is possible to find the true statistics for the MDP via a fixed-point iteration scheme based on the operator $\BellmanOp{\policy}{}$. In contrast, if a collection of statistics $s_{1:K}$ is \emph{not} Bellman closed, there is no Bellman equation relating the statistics of the return distributions, and consequently it is not possible to learn the statistics \emph{exactly} using dynamic programming in a self-contained way; the set of statistics must either be enlarged to make it Bellman closed, or an imputation strategy can be used to perform backups as described in Section~\ref{sec:imputationstrategies}.

An important class of Bellman closed sets of statistics are given in the following result \citep{sobel1982variance,lattimore2012pac}.
\begin{restatable}{lemma}{momentsAreClosed}\label{lem:momentsAreClosed}
    For each $\numstats \in \mathbb{N}$, the set of statistics consisting of the first $\numstats$ moments is Bellman closed.
\end{restatable}
The next result shows that across a wide range of statistics, collections of moments are effectively the only finite sets of statistics that are Bellman closed; the proof relies on a result of \citet{engert1970finite} which characterises finite-dimensional vector spaces of measurable functions closed under translation.
\begin{restatable}{theorem}{classifyBellmanClosed}\label{thm:classifyBellmanClosed}
    The only finite sets of statistics of the form $s(\gendist) = \mathbb{E}_{\genrv \sim \gendist}\!\left\lbrack h(\genrv) \right\rbrack$ that are Bellman closed are given by collections of statistics $\statistic_1, \ldots, \statistic_\numstats : \mathscr{P}(\mathbb{R}) \rightarrow \mathbb{R}$ with the property that the linear span $\{ \sum_{\statix=0}^\numstats \alpha_\statix s_\statix | \alpha_\statix \in \mathbb{R}\ \forall \statix \}$ is equal to the linear span of the set of moment functionals $\{\mu \mapsto \mathbb{E}_{\genrv \sim \gendist}\!\left\lbrack \genrv^l \right\rbrack | l=0,\ldots,L\}$, for some $L \leq \numstats$, where $s_0$ is the constant functional equal to $1$.
\end{restatable}
We believe this to be an important novel result, which helps to highlight how rare it is for statistics to be Bellman closed. One important corollary of Theorem~\ref{thm:classifyBellmanClosed}, given the characterisation of CDRL as learning expectations of return distributions in Lemma~\ref{lem:CDRLStats}, is that the sets of statistics learnt in CDRL are \emph{not} Bellman closed. A similar result holds for QDRL, and we record these facts in the following result. 
\begin{restatable}{lemma}{notClosed}\label{lem:notclosed}
    The sets of statistics learnt under (i) CDRL, and (ii) QDRL, are not Bellman closed.
\end{restatable}
The immediate upshot of this is that in general, the \emph{learnt} values of statistics in distributional RL algorithms need not correspond exactly to the true underlying values for the MDP (even in tabular settings), as the statistics propagated through DRL dynamic programming updates are not sufficient to determine the statistics we seek to learn. This inexactness was noted specifically for CDRL and QDRL in the original papers \citep{C51,QRDQN}.
In this paper, our analysis and experiments confirm that these artefacts arise even with tabular agents in fully-observed domains, thus representing intrinsic properties of the distributional RL algorithms concerned.
However, empirically the distributions learnt by these algorithms are often accurate. In the next section, we provide theoretical guarantees that describe this phenomenon quantitatively.

\subsection{Approximate Bellman closedness}\label{subsec:approximateBellmanClosed}

In light of the results on Bellman closedness in Section~\ref{subsec:bellmanclosed}, we might ask in what sense the values of the statistics learnt by DRL algorithms relate to the corresponding  true underlying values for the MDP concerned. 
A key task in this analysis is to formalise the notion of \emph{low approximation error} in DRL algorithms that seek to learn collections of statistics that are not Bellman closed. Perhaps surprisingly, in general it is not possible to simultaneously achieve low approximation error on all statistics in a non-Bellman closed set; we give several examples for CDRL and QDRL to this end in Appendix Section~\ref{sec:additionalTheory}.

Due to the fact that it is in general not possible to learn statistics uniformly well, we formalise the notion of approximate closedness in terms of the \emph{average} approximation error across a collection of statistics, as described below.

\begin{definition}[\textbf{Approximate Bellman closedness}]\label{def:approxbc}
    A collection of statistics $\statistic_1,\ldots,\statistic_\numstats$, together with an imputation strategy $\Psi$, are said to be $\varepsilon$\emph{-approximately Bellman closed} for a class $\mathcal{M}$ of MDPs if, for each MDP $M=(\statespace,\actionspace,p,\gamma,\mathcal{R})$ in $\mathcal{M}$ and every policy $\pi \in \mathscr{P}(\actionspace)^{\statespace}$, we have
    \begin{align*}
        \sup_{(\state, \action) \in \statespace \times \actionspace} \frac{1}{\numstats} \sum_{\statix=1}^\numstats |\statistic_\statix(\returndist{\policy}(\state, \action)) - \approxstatistic_\statix(\state, \action)  | \leq \varepsilon \, ,
    \end{align*}
    where $\approxstatistic_\statix(\state, \action)$ denotes the learnt value of the statistic $\statistic_\statix$ for the return distribution at the state-action pair $(\state, \action) \in \statespace \times \actionspace$.
\end{definition}

We can now study the approximation errors of CDRL and QDRL in light of this new concept. Whilst the analysis in Section~\ref{subsec:bellmanclosed} shows that CDRL and QDRL necessarily induce some approximation error due to lack of Bellman closedness, the following results reassuringly show that the approximation error can be made arbitrarily small by increasing the number of learnt statistics.

\begin{restatable}{theorem}{CDRLappBellmanClosed}\label{thm:CDRLappBellmanClosed}
    Consider the class $\mathcal{M}$ of MDPs with a fixed discount factor $\gamma \in [0,1)$, and immediate reward distributions supported on $[-\supportbound, \supportbound]$. The set of statistics and imputation strategy corresponding to CDRL with evenly spaced bin locations at $-\supportbound/(1-\gamma) = z_1 < \cdots < z_\numstats = \supportbound/(1-\gamma)$ is $\varepsilon$-approximately Bellman closed for $\mathcal{M}$, where $\varepsilon = \frac{\gamma}{2(1-\gamma)(K-1)}$.
\end{restatable}

\begin{restatable}{theorem}{QDRLappBellmanClosed}\label{thm:QDRLappBellmanClosed}
Consider the class of MDPs $\mathcal{M}$ with a fixed discount factor $\gamma \in [0,1)$, and immediate reward distributions supported on $[-\supportbound, \supportbound]$. Then the collection of quantile statistics $\statistic_\statix(\gendist) = F^{-1}_{\gendist}(\frac{2\statix-1}{2\numstats})$ for $\statix=1,\ldots,\numstats$, together with the standard QDRL imputation strategy, is $\varepsilon$-approximately Bellman closed for $\mathcal{M}$, where $\varepsilon = \frac{2\supportbound(5 - 2\gamma)}{(1-\gamma)^2 K}$.
\end{restatable}

Both of these extend existing analyses for CDRL and QDRL. In particular, Theorem~\ref{thm:CDRLappBellmanClosed} improves on the bound of \citet{AnalysisCDRL}, and Theorem~\ref{thm:QDRLappBellmanClosed} is the first approximation result for QDRL; existing results dealt solely with contraction mappings under $W_\infty$ \citep{QRDQN}.

\subsection{Mean consistency}\label{sec:meanConsistency}
So far, our discussion has been focused around \emph{evaluation}. For \emph{control}, it is important to correctly estimate \emph{expected returns}, so that accurate policy improvement can be performed. We analyse to what extent expected returns are correctly learnt in existing DRL algorithms in the following result. 
The result for CDRL has been shown previously \citep{AnalysisCDRL,lyle2019comparative}, but our proof here gives a new perspective in terms of statistics.

\begin{restatable}{lemma}{CDRLMeanQDRLNoMean}\label{lem:CDRLMeanQDRLNoMean}
    (i) Under CDRL updates using support locations $z_1 < \cdots < z_K$, if all approximate reward distributions have support bounded in $[z_1,z_\numstats]$, expected returns are exactly learnt.
    (ii) Under QDRL updates, expected returns are not exactly learnt.
\end{restatable}

Importantly, for EDRL, as long as the $1/2$-expectile (i.e. the mean) is included in the set of statistics, expected returns are learnt exactly; we return to this point in Section \ref{sec:edrlTabularControl}.

%% file: experiments.tex
\section{Experimental results}\label{sec:experiments}

We first present results with a tabular version of EDRL to illustrate and expand upon the theoretical results presented in Sections \ref{sec:statsAndSamples} and \ref{sec:analysis}.
We then combine the EDRL update with a DQN-style architecture to create a novel deep RL algorithm (ER-DQN), and evaluate performance on the Atari-57 environments. We give full details of the architectures used in experiments in Appendix Section \ref{sec:ERDQNArchitecture}.

There are several ways in which the root-finding/optimisation problems \eqref{eq:rootFindingProblem} and \eqref{eq:minProblem} may be solved in practice. In our experiments, we use a SciPy optimisation routine \citep{scipy}.

\begin{figure}
    \centering
    \includegraphics[keepaspectratio, width=.4\textwidth]{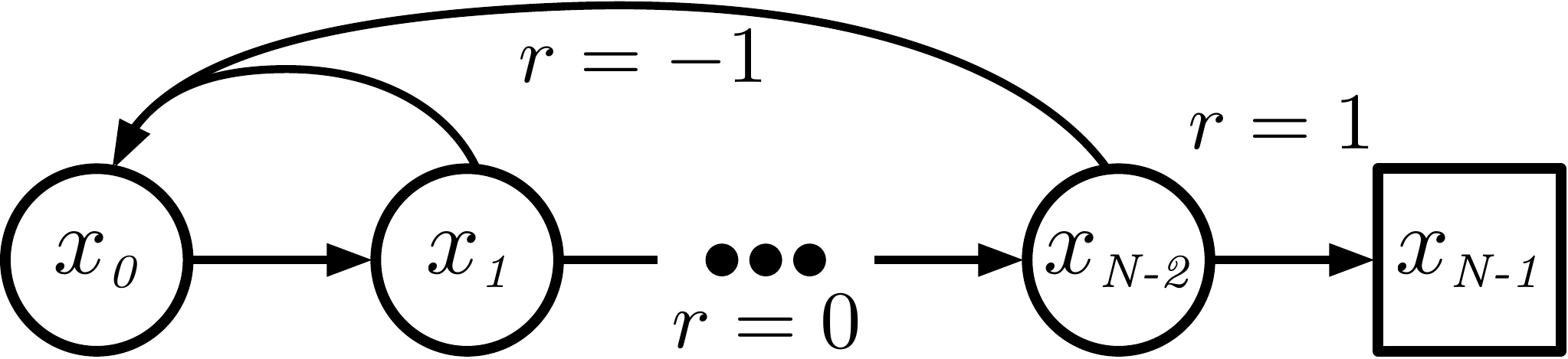}
    \caption{An illustration of the $N$-Chain environment.}
    \label{fig:nchainEnvironment}
\end{figure}

\subsection{Tabular policy evaluation}\label{sec:edrlTabular}
We empirically validate that EDRL, which uses a sample imputation strategy, better approximates the true expectiles of a policy's return distribution as compared to the naive approach described in Section \ref{sec:expectilemotivation}. We then show that the same is true for a variant of QDRL. 

We use a variant of the classic $N$-Chain domain (see Figure \ref{fig:nchainEnvironment}).
This environment is a one-dimensional chain of length $N$ with two possible actions at each state: (i) \texttt{forward}, which moves the agent right by one step with probability 0.95 and to $x_0$ with probability 0.05, and \texttt{backward}, which moves the agent to $x_0$ with probability 0.95 and one step to the right with probability 0.05.
The reward is $-1$ when transitioning to the leftmost state, $+1$ when transitioning to the rightmost state, and zero elsewhere. Episodes begin in the leftmost state and terminate when the rightmost state is reached. The discount factor is $\gamma=0.99$.
For an $N$-Chain with length 15, we compute the return distribution of the optimal policy $\pi^*$ which selects the \texttt{forward} action at each state. This environment formulation induces an increasingly multimodal return distribution under the policy as the distance from the goal state increases. We compute the ground truth start state expectiles from the empirical distribution of 1,000 Monte Carlo rollouts under the policy $\pi^*$.

\textbf{EDRL.} We ran two DRL algorithms on this $N$-Chain environment: (i) EDRL, using a SciPy optimisation routine to impute target samples at each step; and (ii) $\textit{EDRL-Naive}$, using the update described in Section \ref{sec:expectilemotivation}.
We learned $\{1, 3, 5, 7, 9\}$ expectiles, set the learning rate to $\alpha = 0.05$, and performed 30,000 training steps. 

\begin{figure}
    \centering
    \includegraphics[keepaspectratio, width=.42\textwidth]{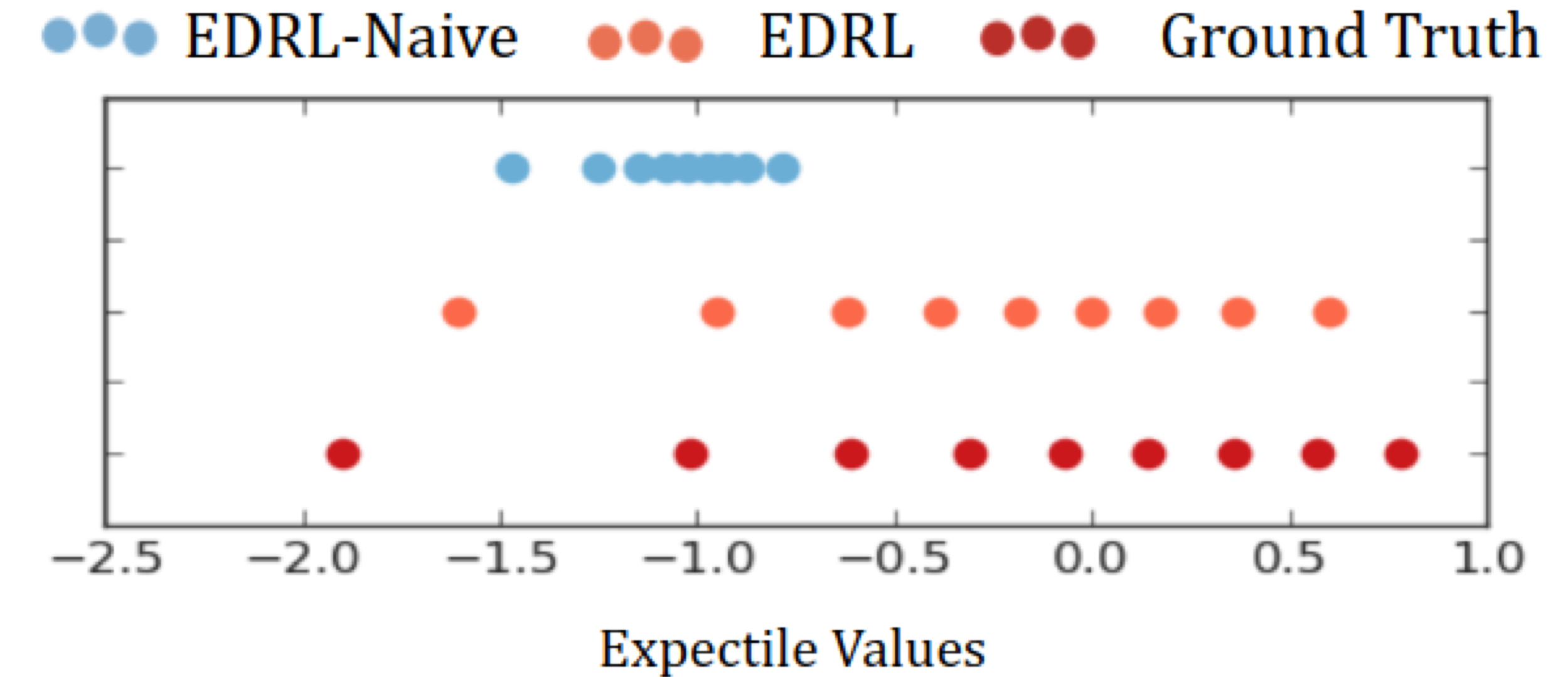}
    \caption{Expectiles for state $x_0$ of the $15$-Chain under policy $\pi^*$.}
    \label{fig:nchainExpectileCollapse}
    \vspace{-0.2cm}
\end{figure}

\begin{figure}
    \centering
    \includegraphics[keepaspectratio, width=.45\textwidth]{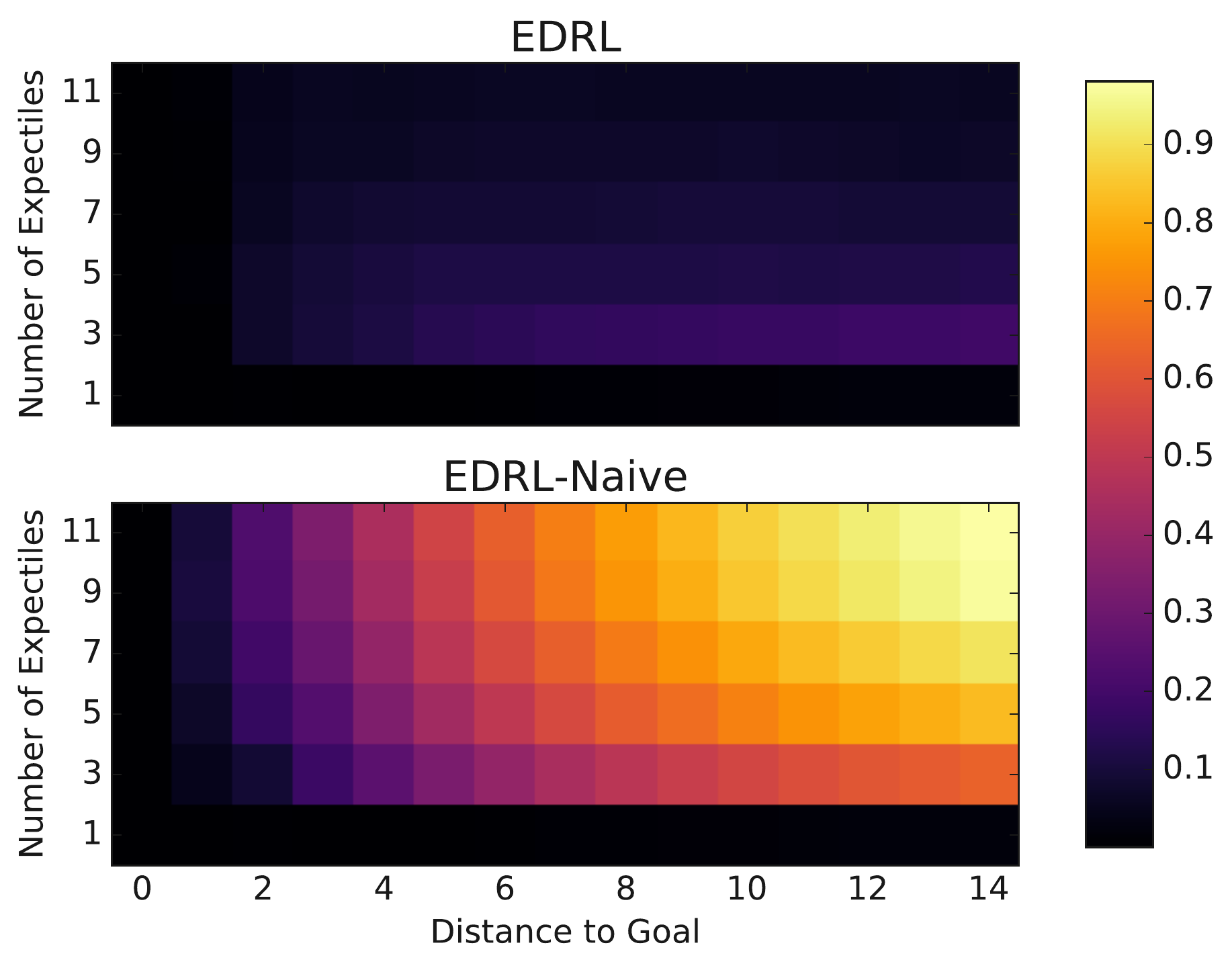}
    \caption{Expectile estimation error for varying numbers of learned expectiles and different $N$-Chain lengths.}
    \label{fig:nchainExpectileError}
\end{figure}

\begin{figure}
    \centering
    \includegraphics[keepaspectratio, width=.45\textwidth]{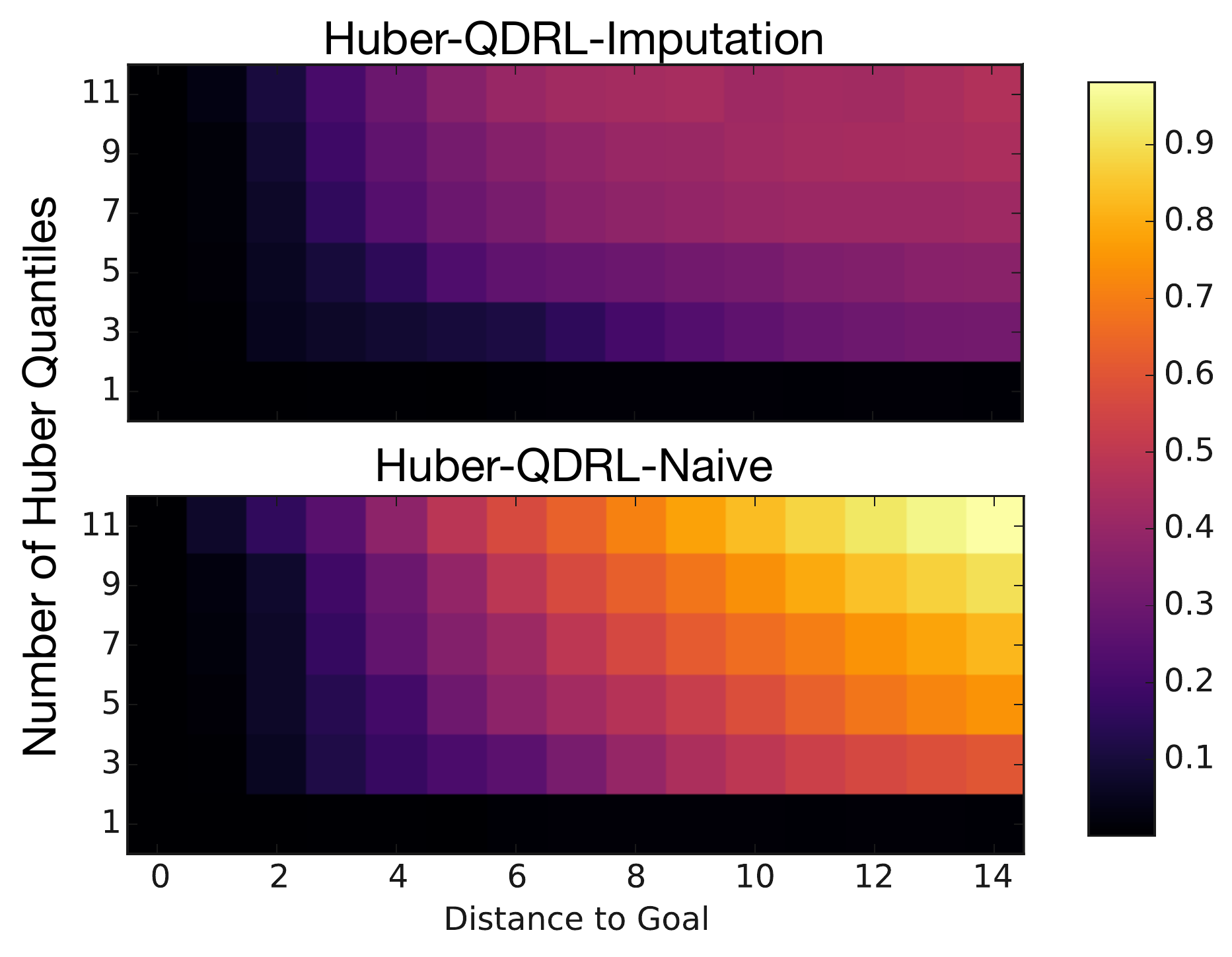}
    \caption{Huber quantile estimation error for varying numbers of learned Huber quantiles at different distances to the goal state. The environment is an $N$-Chain with $N=15$.}
    \vspace{-0.2cm}
    \label{fig:nchain_huber_quantile_error}
\end{figure}

In Figure \ref{fig:nchainExpectileCollapse} we illustrate the collapse of the start state expectiles learned by the EDRL-Naive algorithm with $9$ expectiles, which leads to high expectile estimation error, measured as in Definition~\ref{def:approxbc}.
In Figure \ref{fig:nchainExpectileError}, we show that this error grows as both the distance to the goal state and number of expectiles learned increase. In contrast, under EDRL these errors are much lower
this error remains relatively low for varying numbers of expectiles and distances to the goal with EDRL.
In Appendix \ref{sec:examples}, we illustrate that this observation generalises to other
return distributions in the $N$-Chain.

\textbf{QDRL.} In practical implementations, QDRL often minimises the Huber-quantile loss
\begin{align}\label{eq:huberquantile}
    \argmin_{q \in \mathbb{R}} \mathbb{E}_{\genrv \sim \gendist}\!\left\lbrack (\tau \mathbbm{1}_{\genrv > q} \!+\! (1\!-\!\tau)\mathbbm{1}_{\genrv < q}) H_\kappa(\genrv \!-\! q) \right\rbrack  ,
\end{align}
rather than the quantile loss \eqref{eq:qr} for numerical stability,
where $H_\kappa$ is the Huber loss function with width parameter $\kappa$, as in \citet{QRDQN} (we set $\kappa = 1$). As with naive EDRL, simply replacing the quantile regression loss in QDRL with Expression \eqref{eq:huberquantile} conflates samples and statistics, leading to worse approximation of the distribution.
We propose a new algorithm for learning Huber quantiles, \emph{Huber-QDRL-Imputation}, that incorporates an imputation strategy
by solving an optimisation problem analogous to \eqref{eq:minProblem} in the case of the Huber quantile loss. In Figure~\ref{fig:nchain_huber_quantile_error}, we compare this to \emph{Huber-QDRL-Naive}, the standard algorithm for learning Huber quantiles, on the $N$-chain environment. As in the case of expectiles, the Huber quantile estimation error is vastly reduced when using an imputation strategy.

\subsection{Tabular control}\label{sec:edrlTabularControl}
\begin{figure}
    \centering
    \includegraphics[keepaspectratio, width=.45\textwidth]{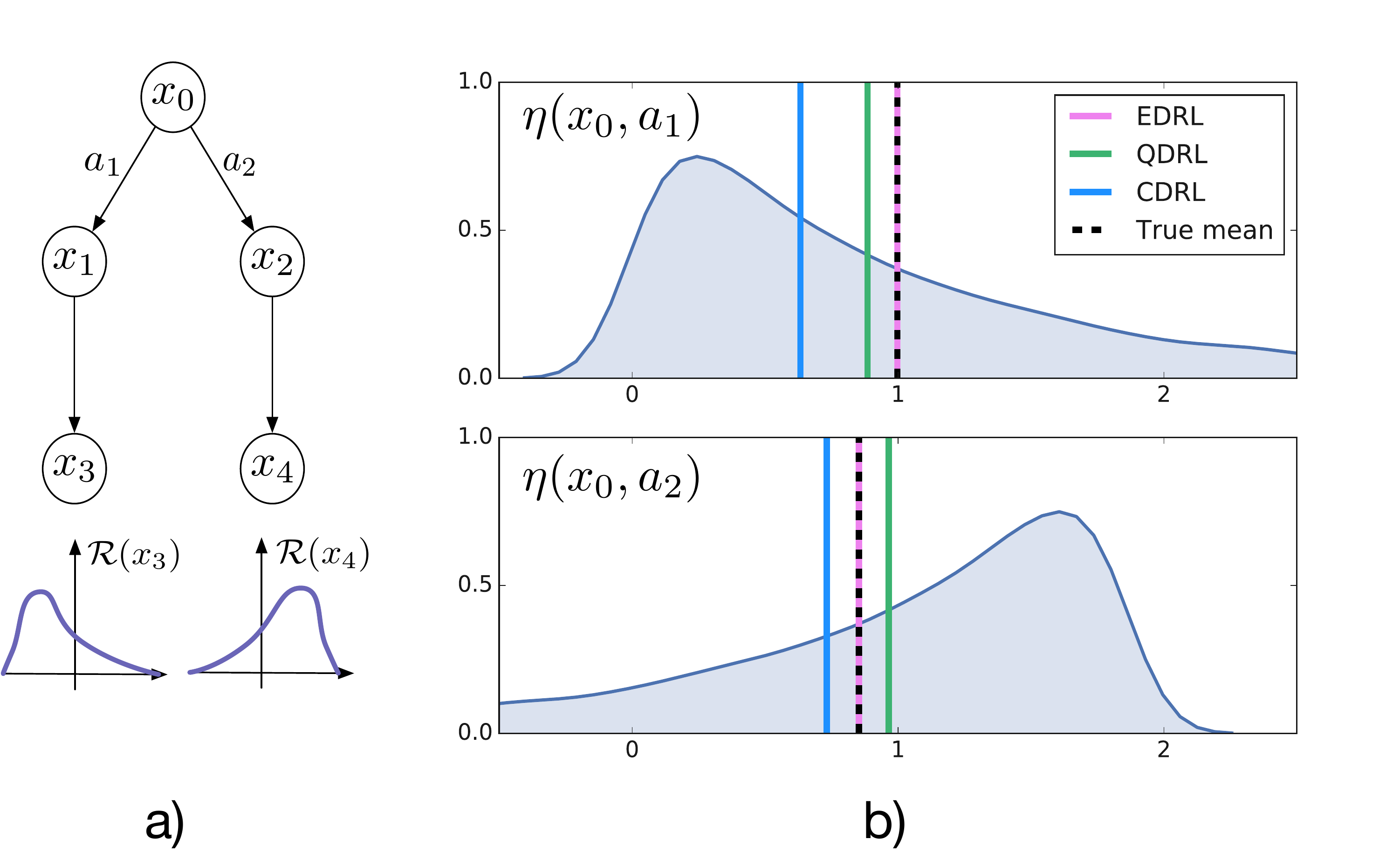}
    \caption{(a) 5-state MDP, reward is zero everywhere except at the terminal states $x_3$ and $x_4$ which have stochastic rewards. (b) We show the true return distributions $\approxdist(\state_0, \action_1)$ and $\approxdist(\state_0, \action_2)$, and the expected returns estimated by CDRL, QDRL, and EDRL.}
    \vspace{-0.3cm}
    \label{fig:toy_mdp_control}
\end{figure}
In Section~\ref{sec:meanConsistency} we argued for the importance of mean consistency. In Figure~\ref{fig:toy_mdp_control}a we give a simple, five state, MDP in which the learned control policy is directly affected by mean consistency. At start state $x_0$ the agent has the choice of two actions, leading down two paths and culminating in two different reward distributions. The rewards at terminal states $x_3$ and $x_4$ are sampled from (shifted) exponential distributions with densities $e^{-\lambda}$ ($\lambda \geq 0$) and $e^{\lambda+1.85}$ ($\lambda \leq 1.85$), respectively. Transitions are deterministic, and $\gamma = 1$. For CDRL, we take bin locations at $(z_1,z_2,z_3) = (0,1,2)$.

Figure~\ref{fig:toy_mdp_control}b shows the true return distributions, their expectations, and the means estimated by CDRL, QDRL and EDRL. Due to a lack of mean consistency both CDRL and QDRL learn a sub-optimal greedy policy.
For CDRL, this is due to the true return distributions having support outside $[0,2]$, and for QDRL, this is due to the quantiles not capturing tail behaviour. In contrast, EDRL correctly learns the means of both return distributions, and so is able to act optimally.

\subsection{Expectile regression DQN}\label{sec:atari}
To demonstrate the effectiveness of EDRL at scale, we combine the EDRL update in Algorithm \ref{alg:EDRL} with the architecture of QR-DQN to obtain a new deep RL agent, expectile regression DQN (ER-DQN).
Precise details of the architecture, training algorithm, and environments are given in Appendix Section \ref{sec:erdqndetails}.
We evaluate ER-DQN on a suite of 57 Atari games using the Arcade Learning Environment \citep{bellemare2013arcade}. In Figure \ref{fig:atariresults}, we plot mean and median human normalised scores for ER-DQN with 11 atoms, and compare against DQN, QR-DQN (which learns 200 Huber quantile statistics), and a naive implementation of ER-DQN that doesn't use an imputation strategy, learning 201 expectiles. All methods were re-run for this paper, and results were averaged over 3 seeds. In practice, we found that with 11 expectiles, ER-DQN already offers strong performance relative to these other approaches, and that with this number of statistics, the additional training overhead due to the SciPy optimiser calls is low.

\begin{figure}[t]
    \centering
    \includegraphics[keepaspectratio,width=.45\textwidth]{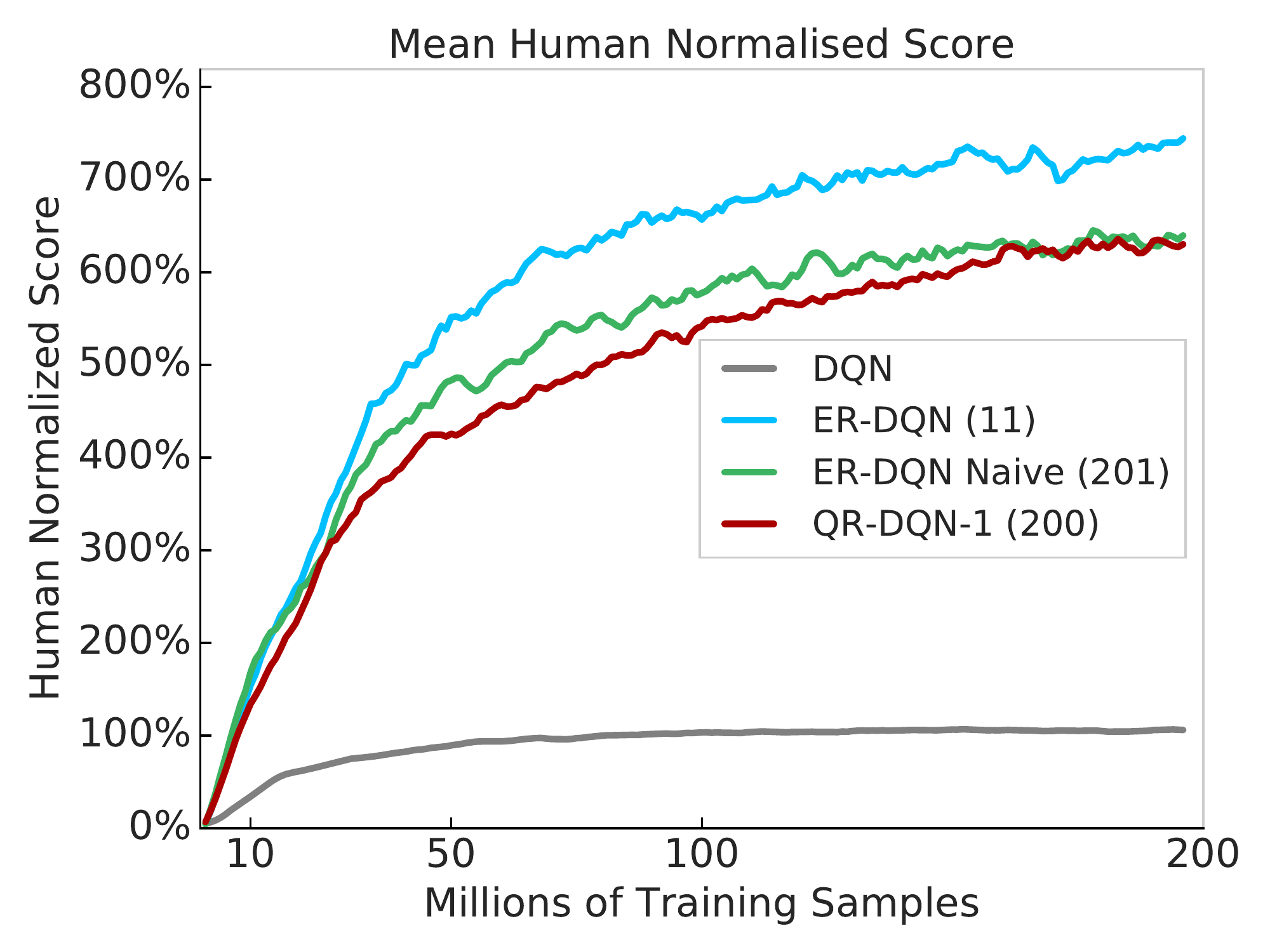}
    \includegraphics[keepaspectratio,width=.45\textwidth]{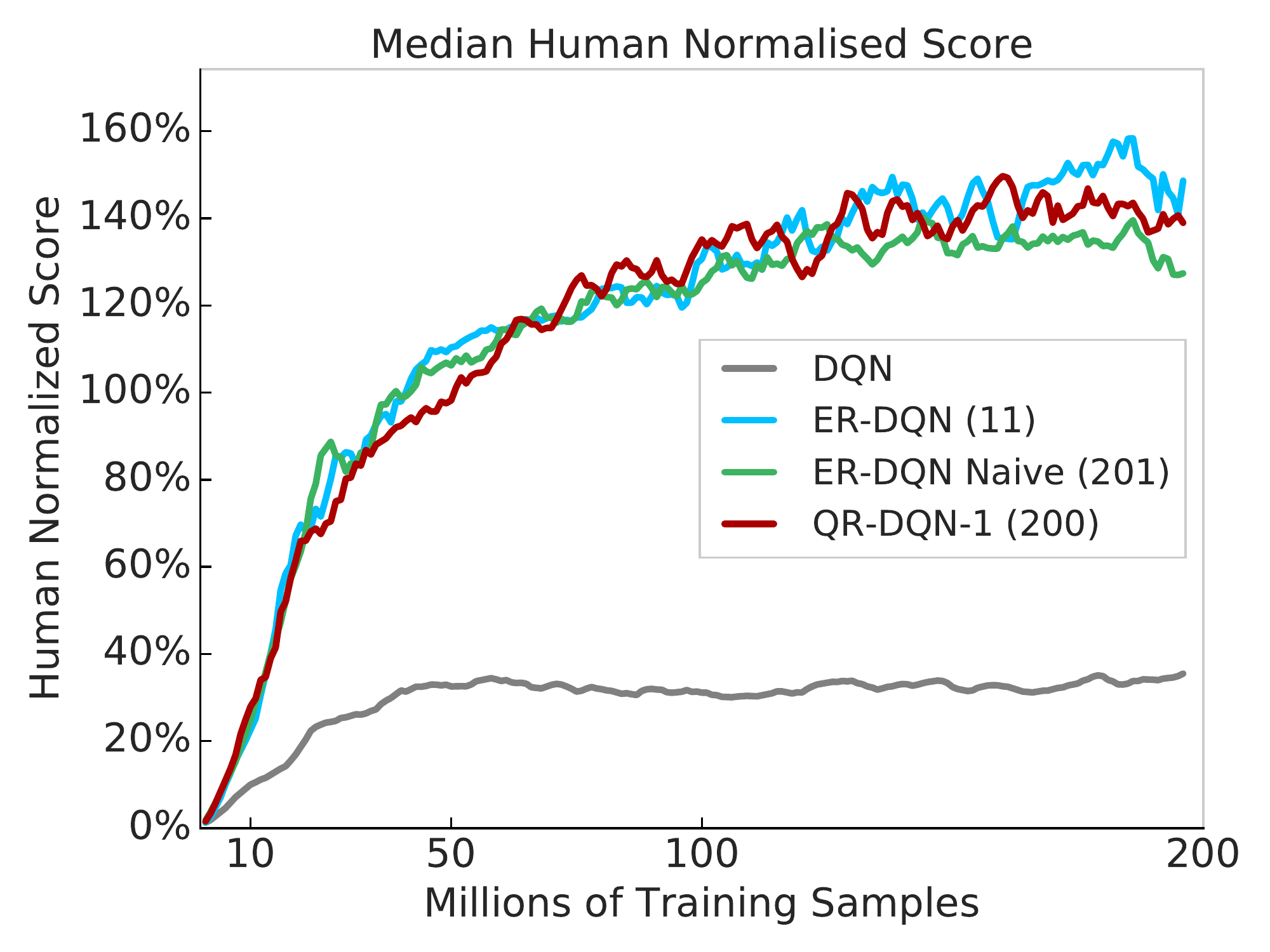}
    \caption{Mean and median human normalised scores across all 57 Atari games. Number of statistics learnt for each algorithm indicated in parentheses.}
    \label{fig:atariresults}
\end{figure}

In terms of mean human normalised score, ER-DQN represents a substantial improvement over both QR-DQN and the naive version of ER-DQN that does not use an imputation strategy. We hypothesise that the mean consistency of EDRL (in contrast to other DRL methods; see Section \ref{sec:meanConsistency}) is partially responsible for these improvements, and leave further investigation of the role of mean consistency in DRL as a direction for future work. We also remark that the performance of ER-DQN shows that there may be significant practical value in applying the framework developed in this paper to other families of statistics. It remains to be seen if the presence of partial observability may induce non-trivial distributions, which could also explain ER-DQN's improved performance in some games. Investigation into the robustness of ER-DQN with regards to the precise imputation strategy used is also a natural question for future work.

%% file: conclusion.tex
\section{Conclusion}\label{sec:conclusion}

We have developed a unifying framework for DRL in terms of statistical estimators and imputation strategies. Through this framework, we have developed a new algorithm, EDRL, as well as proposing algorithmic adjustments to an existing approach. We have also used this framework to define the notion of Bellman closedness, and provided new approximation guarantees for existing algorithms.

This paper also opens up several avenues for future research.
Firstly, the framework of imputation strategies has the potential to be applied to a wide range of collections of statistics, opening up a large space of new algorithms to explore.
Secondly, our analysis has shown that a lack of Bellman closedness necessarily introduces a source of approximation error into many DRL algorithms; it will be interesting to see how this interacts with errors introduced by function approximation.
Finally, we have focused on DRL algorithms that can be interpreted as learning a finite collection of statistics in this paper. One notable alternative is implicit quantile networks \citep{IQN}, which attempt to learn an uncountable collection of quantiles with a finite-capacity function approximator; it will also be interesting to extend our analysis to this setting.

%% file: appendixAlgorithms.tex
\section{Distributional reinforcement learning algorithms}\label{sec:olderalgorithms}

For completeness, we give full descriptions of CDRL and QDRL algorithms in this section, complementing the details given in Section \ref{sec:CDRLQDRLDescriptions}. We also summarise CDRL, QDRL, the exact approach to distributional RL, and our proposed algorithm EDRL, in Figure \ref{fig:all_algs} at the end of this section.

\subsection{The distributional Bellman operator}\label{sec:distBellmanOp}
 In accordance with the distributional Bellman equation \eqref{eq:distBellmanEq}, the distributional Bellman operator $\BellmanOp{\policy}{} : \mathscr{P}(\mathbb{R})^{\statespace \times \actionspace} \rightarrow \mathscr{P}(\mathbb{R})^{\statespace \times \actionspace}$ is defined by \citet{C51} as
\begin{align*}
    (\BellmanOp{\policy}{}\eta)(\state, \action) = \mathbb{E}_\pi\left\lbrack  (f_{\rewardvar_0, \gamma})_\# \eta(\statevar_1, \actionvar_1) |  \statevar_0 \!=\! \state, \actionvar_0 \!=\! \action \right\rbrack \, ,
\end{align*}
for all $\eta \in \mathscr{P}(\mathbb{R})^{\statespace\times\actionspace}$.

\subsection{Categorical distributional reinforcement learning}\label{sec:cdrl}

As described in Section \ref{sec:CDRLQDRLDescriptions}, CDRL algorithms are an approach to distributional RL that restrict approximate distributions to the parametric family of the form $\{ \sum_{\statix=1}^\numstats p_\statix \delta_{z_\statix} | \sum_{\statix=1}^\numstats p_\statix = 1,\ p_\statix \geq 0 \forall \statix \} \subseteq \mathscr{P}(\mathbb{R})$, where $z_1 < \cdots < z_\numstats$ are an evenly spaced, fixed set of supports.
For evaluation of a policy $\policy : \statespace \rightarrow \mathscr{P}(\actionspace)$, given a collection of approximations $(\approxdist(\state, \action) | (\state, \action) \in \statespace \times \actionspace)$, the approximation at $(\state, \action) \in \statespace\times\actionspace$ is updated according to: 
\begin{align*}
    \approxdist(\state, \action) \leftarrow \Pi_\mathcal{C} \mathbb{E}_\pi\left\lbrack (f_{\rewardvar_0, \gamma})_\# \approxdist(\statevar_1, \actionvar_1) |  \statevar_0 \!=\! \state, \actionvar_0 \!=\! \action \right\rbrack \, .
\end{align*}
Here, $\Pi_\mathcal{C} : \mathscr{P}(\mathbb{R}) \rightarrow \mathscr{P}(\{z_1,\ldots,z_\numstats\})$ is a projection operator defined for a single Dirac delta as 
\begin{align}\label{eq:cramerProj}
    \Pi_\mathcal{C}(\delta_{w}) =
    \begin{cases}
        \delta_{z_1} & w \leq z_1 \\
        \frac{w - z_{\statix+1}}{z_{\statix} - z_{\statix+1}}\delta_{z_\statix} + \frac{z_{\statix} - w}{z_{\statix} - z_{\statix+1}} \delta_{\statix+1}  & z_{\statix} \leq w \leq z_{\statix+1} \\
        \delta_{z_K} & w \geq z_\numstats \, ,
    \end{cases} 
\end{align}
and extended affinely and continuously. In the language of operators, the CDRL update may be neatly described as $\approxdist \leftarrow \Pi_\mathcal{C} \BellmanOp{\policy}{} \approxdist{}$, where we abuse notation by interpreting $\Pi_\mathcal{C}$ as an operator on collections of distributions indexed by state-action pairs, applying the transformation in Expression \eqref{eq:cramerProj} to each distribution.
The supremum-Cram\'er distance is defined as 
\begin{align*}
    \overline{\ell}_2(\returndist{1}, \returndist{2}) = \sup_{(\state, \action) \in \statespace\times\actionspace} \ell_2(\returndist{1}(\state, \action), \returndist{2}(\state, \action)) = \sup_{(\state, \action) \in \statespace\times\actionspace} \Big( \int_{\mathbb{R}} | F_{\returndist{1}(\state, \action)}(t) - F_{\returndist{2}(\state, \action)}(t) |^2 \mathrm{d}t \Big)^{\frac{1}{2}} \, .
\end{align*}
for all $\returndist{1}, \returndist{2} \in \mathscr{P}(\mathbb{R})^{\statespace \times \actionspace}$, where for any $\mu \in \mathscr{P}(\mathbb{R})$, $F_\mu$ denotes the CDF of $\mu$. The operator $\Pi_\mathcal{C} \BellmanOp{\policy}{}$ is a $\sqrt{\gamma}$-contraction in the supremum-Cram\'er distance, and so by the contraction mapping theorem, repeated CDRL updates converge to a unique limit point, regardless of the initial approximate distributions.
For more details on these results and further background, see \citet{C51,AnalysisCDRL}.

\textbf{Stochastic approximation.} The update $\approxdist \leftarrow \Pi_\mathcal{C} \BellmanOp{\policy}{} \approxdist$ is typically not computable in practice, due to unknown/intractable dynamics. An unbiased approximation to $(\BellmanOp{\policy}{}\approxdist)(\state, \action)$ may be obtained by interacting with the environment to obtain a transition $(\state, \action, \reward, \state^\prime, \action^\prime)$, and computing the target
\begin{align*}
    (f_{\reward, \gamma})_\# \approxdist(\state^\prime, \action^\prime) \, .
\end{align*}
It can be shown \citep{AnalysisCDRL} that the following is an unbiased estimator for the CDRL update $(\Pi_\mathcal{C} \BellmanOp{\policy}{} \approxdist)(\state, \action)$:
\begin{align*}
    \Pi_\mathcal{C} (f_{\reward, \gamma})_\# \approxdist(\state^\prime, \action^\prime) \, .
\end{align*}
Finally, the current estimate $\approxdist(\state, \action)$ can be moved towards the stochastic target by following the (semi-)gradient of some loss, in analogy with semi-gradient methods in classical RL. \citet{C51} consider the KL loss
\begin{align*}
    \mathrm{KL}(\Pi_\mathcal{C} (f_{\reward, \gamma})_\# \approxdist(\state^\prime, \action^\prime)\ ||\ \approxdist(\state, \action)) \, ,
\end{align*}
and update $\approxdist(\state, \action)$ by taking the gradient of the loss through the second argument with respect to the parameters $p_{1:\numstats}(\state, \action)$. Other losses, such as the Cram\'er distance, may also be considered \citep{AnalysisCDRL}.

\textbf{Control.} 
All variants of CDRL for evaluation may be modified to become control algorithms. This is achieved by adjusting the distribution of the action $A_1$ in the backup in an analogous way to classical RL algorithms. Instead of having $\actionvar_1 \sim \pi(\cdot | \statevar_1)$, we instead select $\actionvar_1$ based on the currently estimated expected returns for each of the actions at the state $\statevar_1$. For Q-learning-style algorithms, the action corresponding to the highest estimated expected return is selected:
\begin{align*}
    \actionvar_1 = \argmax_{a \in \actionspace} \mathbb{E}_{\genrv \sim \approxdist(\statevar_1, \action)}\!\left\lbrack Z \right\rbrack \, .
\end{align*}
However, other choices are possible, such as SARSA-style $\varepsilon$-greedy action selection.

\subsection{Quantile distributional reinforcement learning}\label{sec:qdrl}

As described in Section \ref{sec:CDRLQDRLDescriptions}, QDRL algorithms are an approach to distributional RL that restrict approximate distributions to the parametric family of the form $\{ \frac{1}{\numstats}\sum_{\statix=1}^\numstats \delta_{z_\statix} | z_{1:\numstats} \in \mathbb{R}^\numstats \} \subseteq \mathscr{P}(\mathbb{R})$.
For evaluation of a policy $\policy : \statespace \rightarrow \mathscr{P}(\actionspace)$, given a collection of approximations $(\approxdist(\state, \action) | (\state, \action) \in \statespace \times \actionspace)$, the approximation at $(\state, \action) \in \statespace\times\actionspace$ is updated according to: 
\begin{align*}
    \approxdist(\state, \action) \leftarrow \Pi_{W_1} \mathbb{E}_\pi\left\lbrack (f_{\rewardvar_0, \gamma})_\# \approxdist(\statevar_1, \actionvar_1) |  \statevar_0 \!=\! \state, \actionvar_0 \!=\! \action \right\rbrack \ , .
\end{align*}
Here, $\Pi_{W_1} : \mathscr{P}(\mathbb{R}) \rightarrow \mathscr{P}(\mathbb{R})$ is a projection operator defined by
\begin{align*}
    \Pi_\mathcal{C}(\gendist) = \frac{1}{\numstats} \sum_{\statix=1}^\numstats \delta_{F_\mu^{-1}(\tau_\statix)} \, ,
\end{align*}
where $\tau_\statix = \frac{2\statix-1}{2\numstats}$, and $F_\gendist$ is the CDF of of $\gendist$. As noted in Section \ref{sec:CDRLQDRLDescriptions}, $F_\mu^{-1}(\tau)$ may also be characterised as the minimiser (over $q \in \mathbb{R}$) of the quantile regression loss $\mathrm{QR}(q; \mu, \tau) = \mathbb{E}_{\genrv \sim \gendist}\left\lbrack \left\lbrack \tau \mathbbm{1}_{\genrv > q} + (1 - \tau) \mathbbm{1}_{\genrv \leq q} \right\rbrack |\genrv - q| \right\rbrack$; this perspective turns out to be crucial in deriving a stochastic approximation version of the algorithm.

\textbf{Stochastic approximation.} As for CDRL, the update $\approxdist \leftarrow \Pi_{W_1} \BellmanOp{\policy}{} \approxdist$ is typically not computable in practice, due to unknown/intractable dynamics. Instead, a stochastic target may be computed by using a transition $(\state, \action, \reward, \state^\prime, \action^\prime)$, and updating each atom location $z_\statix(\state, \action)$ at the current state-action pair $(\state, \action)$ by following the gradient of the QR loss:
\begin{align*}
    \nabla_q \mathrm{QR}(q; (f_{\reward, \gamma})_\# \approxdist(\state^\prime, \action^\prime), \tau_\statix) \big|_{q = z_\statix(\state, \action)} \, .
\end{align*}
Because the $\mathrm{QR}$ loss is affine in its second argument, this yields an unbiased estimator of the true gradient
\begin{align*}
    \nabla_q \mathrm{QR}(q;
         (\BellmanOp{\policy}{} \approxdist)(\state, \action), 
         \tau_\statix) \big|_{q = z_\statix(\state, \action)} \, .
\end{align*}

\textbf{Control.} The methods for evaluation described above may be modified to yield control methods in exactly the same as described for CDRL in Section \ref{sec:cdrl}.

\subsection{Quantiles versus expectiles}\label{sec:qvse}

\begin{wrapfigure}{r}{0.34\textwidth} 
    \centering
    \vspace{-0.8cm}
    \includegraphics[keepaspectratio,width=.34\textwidth]{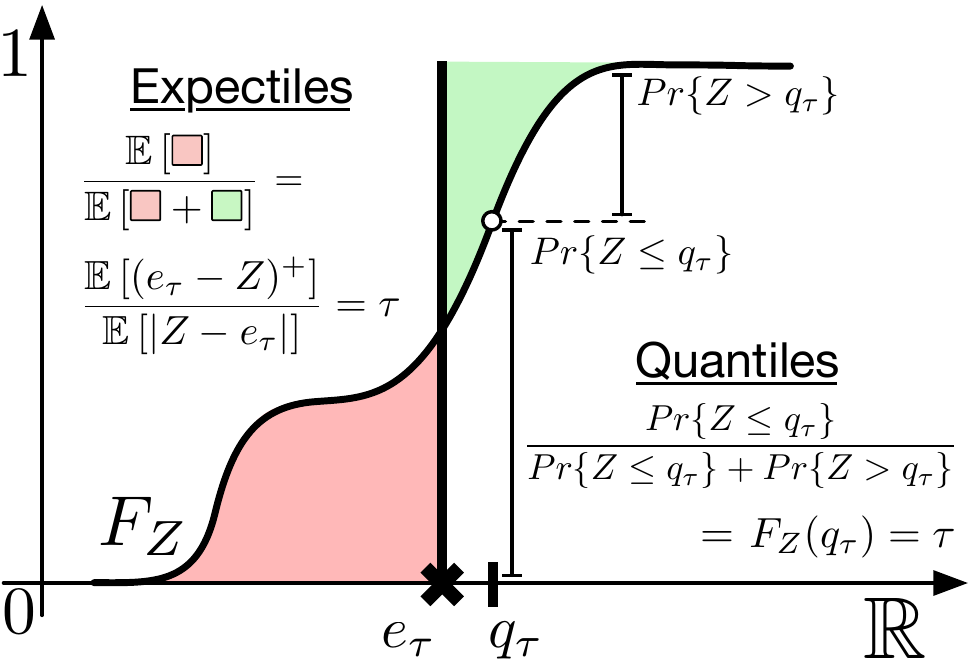}
    \vspace{-0.6cm}
    \caption{Diagram illustrating the similarities and differences of quantiles and expectiles.}
    \vspace{-1.5cm}
    \label{fig:quantvsexpect}
\end{wrapfigure}
Quantiles of a distribution are given by the inverse of the cumulative distribution function. As such, they fundamentally represent threshold values for the cumulative probabilities. That is, the quantile at $\tau$, $q_\tau$, is greater than or equal to $\tau \times 100\%$ of the outcome values. In contrast, expectiles also take into account the \emph{magnitude} of outcomes; the expectile at $\tau$, $e_\tau$, is such that the expectation of the deviations below $e_\tau$ of the random variable $Z$ is equal to $\frac{\tau}{1-\tau}$ of the expectation of the deivations above $e_\tau$. 
We illustrate these points in Figure~\ref{fig:quantvsexpect}.

\begin{figure}
    \centering
    \includegraphics[keepaspectratio,width=\textwidth]{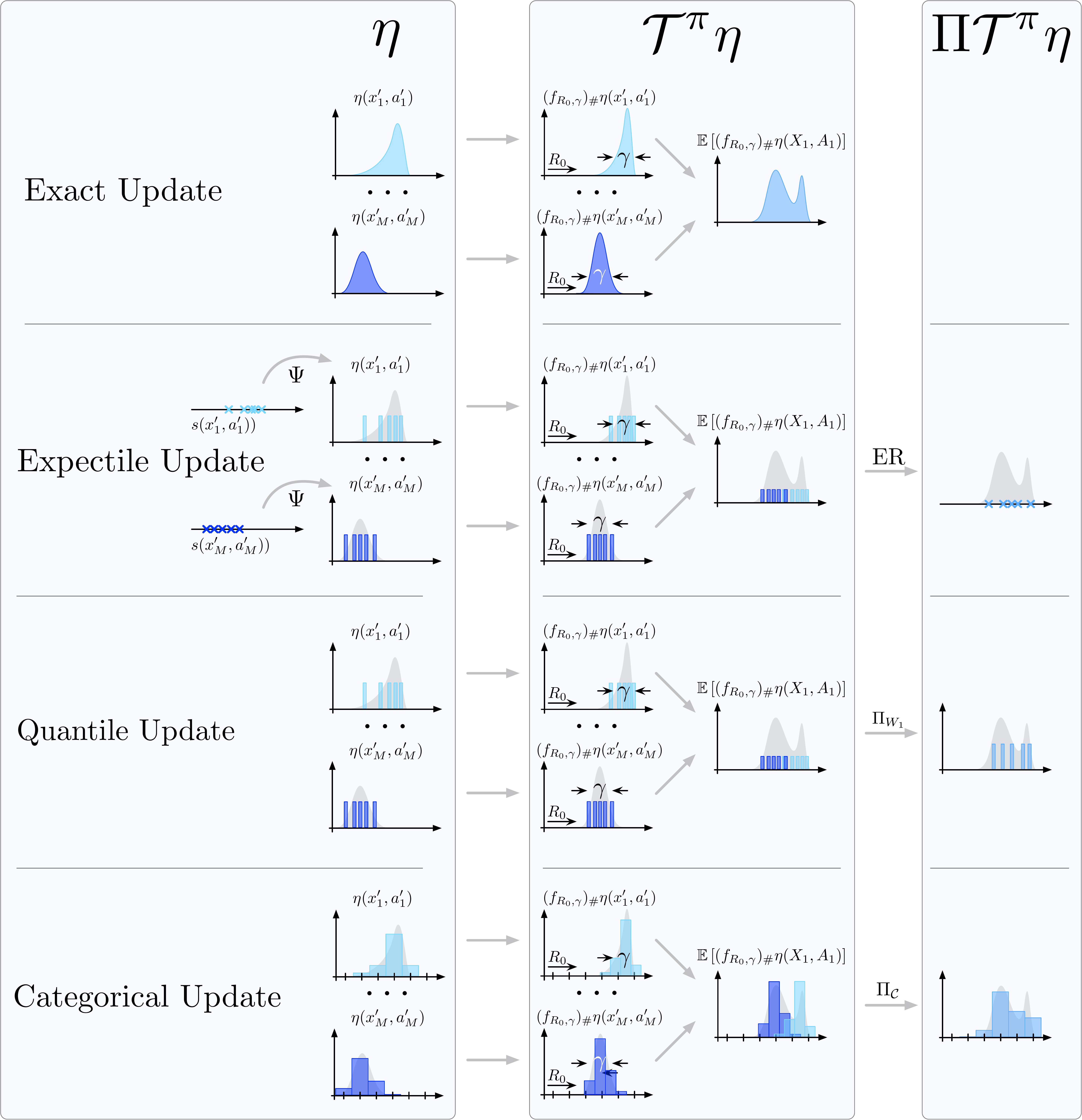}
    \caption{Illustration of distributional RL, with exact updates, expectile updates (EDRL), quantile updates (QDRL), and categorical updates (CDRL).}
    \label{fig:all_algs}
\end{figure}

\newpage

%% file: appendixProofs.tex
\section{Proofs}

\subsection{Proofs of results from Section \ref{sec:statsAndSamples}}

\CDRLStats*

\begin{proof}
    We first observe that the projection operator $\Pi_\mathcal{C}$, defined in Section \ref{sec:cdrl}, preserves each of the statistics $s_{z_1,z_2}, \ldots, s_{z_{\numstats-1}, z_{\numstats}}$, in the sense that for any distribution $\gendist$, we have $s_{z_\statix, z_{\statix+1}}(\gendist) = s_{z_\statix, z_{\statix+1}}(\Pi_\mathcal{C} \gendist)$ for all $\statix=1,\ldots,\numstats$. Secondly, we observe that that the map $\{\sum_{\statix=1}^\numstats p_\statix \delta_{z_\statix} | \sum_{\statix=1}^\numstats p_\statix = 1,\ p_\statix \geq 0 \forall \statix \} \ni \gendist \mapsto (s_{z_1,z_2}(\gendist), \ldots, s_{z_{\numstats-1}, z_\numstats}(\gendist)) \in \mathbb{R}^{\numstats-1}$ is injective; each distribution has a unique vector of statistics. Thus, CDRL can indeed be interpreted as learning precisely the set of statistics $s_{z_1,z_2},\ldots,s_{z_{\numstats-1},z_{\numstats}}$.
\end{proof}

\subsection{Proofs of results from Section \ref{subsec:bellmanclosed}}

\momentsAreClosed*

\begin{proof}
    We begin by introducing notation. Let $\statistic_\statix : \gendist \mapsto \mathbb{E}_{\genrv \sim \gendist}\left\lbrack \genrv^\statix \right\rbrack$ be the $\statix$\textsuperscript{th} moment functional, for $\statix=1,\ldots,\numstats$. We now compute
    \begin{align*}
        \statistic_\statix(\returndist{\policy}(\state, \action))
            & = \mathbb{E}_{\genrv \sim \returndist{\policy}(\state, \action)}\left\lbrack  \genrv^\statix \right\rbrack \\
            & = \sum_{(\state^\prime, \action^\prime) \in \statespace \times \actionspace} \int_{\mathbb{R}} \mathcal{R}(\mathrm{d}r | \state, \action) p(\state^\prime | \state, \action) \policy(\action^\prime | \state^\prime)
                \mathbb{E}_{Z \sim \returndist{\policy}(\state^\prime, \action^\prime)}\left\lbrack (r + \gamma Z)^\statix \right\rbrack \\
            & = \sum_{(\state^\prime, \action^\prime) \in \statespace \times \actionspace} \int_{\mathbb{R}} \mathcal{R}(\mathrm{d}r | \state, \action) p(\state^\prime | \state, \action) \policy(\action^\prime | \state^\prime)
                \sum_{m=0}^k \binom{k}{m} \gamma^{k-m} \mathbb{E}_{Z \sim \returndist{\policy}(\state^\prime, \action^\prime)}\left\lbrack Z^{k-m} \right\rbrack r^{m} \\
            & = \sum_{(\state^\prime, \action^\prime) \in \statespace \times \actionspace} \int_{\mathbb{R}} \mathcal{R}(\mathrm{d}r | \state, \action) p(\state^\prime | \state, \action) \policy(\action^\prime | \state^\prime)
                \sum_{m=0}^k \binom{k}{m} \gamma^{k-m} s_{\statix-m}(\returndist{\policy}(\state^\prime, \action^\prime)) r^{m} \\
            & = \mathbb{E}\left\lbrack 
                \sum_{m=0}^k \binom{k}{m} \gamma^{k-m} s_{\statix-m}(\returndist{\policy}(\statevar_1, \actionvar_1)) \rewardvar_0^{m} \Bigg| \statevar_0 = \state, \actionvar_0 = \action \right\rbrack \, .
    \end{align*}
    Thus, $\statistic_\statix(\returndist{\policy}(\state, \action))$ can be expressed in terms of $\rewardvar_0$ and $\statistic_{1:\numstats}(\returndist{\policy}(\statevar_1, \actionvar_1))$, as required.
\end{proof}

\classifyBellmanClosed*

\begin{proof}

    Suppose $\statistic_1,\ldots, \statistic_\numstats : \mathscr{P}(\mathbb{R}) \rightarrow \mathbb{R}$ form a Bellman closed set of statistical functionals of the form $\statistic_\statix(\gendist) = \mathbb{E}_{\genrv \sim \gendist}\left\lbrack h_\statix(\genrv) \right\rbrack$ for some measurable $h_\statix : \mathbb{R} \rightarrow \mathbb{R}$, for each $\statix = 1,\ldots,\numstats$. Now note that for any MDP $(\statespace, \actionspace, p, \gamma, \mathcal{R})$, we have the following equation:
    \begin{align*}
        \statistic_k(\returndist{\policy}(\state, \action)) = \sum_{(\state^\prime, \action^\prime) \in \statespace \times \actionspace} \int_{\mathbb{R}} \mathcal{R}(\mathrm{d}\reward | \state,\action) p(\state^\prime | \state, \action) \policy(\action^\prime | \state^\prime) \statistic_\statix((f_{\reward, \gamma})_\#\returndist{\policy}(\state^\prime, \action^\prime)) \, ,
    \end{align*}
    for all $(\state, \action) \in \statespace \times \actionspace$, and for each $\statix=1,\ldots,\numstats$. By assumption of Bellman closedness, the right-hand side of this equation may be written as a function of $\mathcal{R}(x, a)$, $\gamma$, and the collection of statistics $(\statistic_{1:\numstats}(\returndist{\policy}(\state^\prime, \action^\prime)) | (\state^\prime, \action^\prime) \in \statespace \times \actionspace )$. Since this must hold across all valid sets of return distributions, it must the case that each $\statistic_\statix((f_{\reward, \gamma})_\#\returndist{\policy}(\state^\prime, \action^\prime))$ may be written as a function of $\reward$, $\gamma$ and $\statistic_{1:\numstats}(\returndist{\policy}(\state^\prime, \action^\prime))$; we will write $\statistic_\statix((f_{\reward, \gamma})_\#\returndist{\policy}(\state^\prime, \action^\prime)) = g(\reward, \gamma, \statistic_{1:\numstats}(\returndist{\policy}(\state^\prime, \action^\prime)))$ for some $g$.
    
    We next claim that $g(\reward, \gamma, \statistic_{1:\numstats}(\returndist{\policy}(\state^\prime, \action^\prime)))$ is affine in $\statistic_{1:\numstats}(\returndist{\policy}(\state^\prime, \action^\prime))$. To see this, note that both $\statistic_\statix((f_{\reward, \gamma})_\#\returndist{\policy}(\state^\prime, \action^\prime))$ and $\statistic_{1:\numstats}(\returndist{\policy}(\state^\prime, \action^\prime))$ are affine as functions of the distribution $\returndist{\policy}(\state^\prime, \action^\prime)$, by assumption on the form of the statistics $\statistic_{1:\numstats}$. Therefore $g(r,\gamma, \cdot)$ too is affine on the (convex) codomain of $s_{1:K}$.
    
    Thus, we have
    \begin{align}
        \mathbb{E}_{\genrv \sim \returndist{\policy}(\state^\prime,\action^\prime)}\left\lbrack h_\statix(\reward + \gamma \genrv) \right\rbrack = a_0(\reward, \gamma) + \sum_{\statix^\prime=1}^\numstats a_{\statix^\prime}(r, \gamma) \mathbb{E}_{\genrv \sim \returndist{\policy}(\state^\prime, \action^\prime)}\left\lbrack h_{\statix^\prime}(\genrv) \right\rbrack \, ,
    \end{align}
    for some functions $a_{0:\numstats} : \mathbb{R} \times [0,1) \rightarrow \mathbb{R}$. 
    By taking $\returndist{\policy}(\state^\prime, \action^\prime)$ to be a Dirac delta at an arbitrary real number, we obtain
    \begin{align}\label{eq:equalAsFunctions}
        h_k(\reward + \gamma x) = a_0(\reward, \gamma) + \sum_{\statix^\prime=1}^\numstats a_{\statix^\prime}(\reward, \gamma) h_{\statix^\prime}(x) \quad \text{ for all } x \in \mathbb{R} \, .
    \end{align}
    In particular, the function $h_\statix(\gamma x)$ lies in the span of the functions $h_1,\ldots,h_\numstats,\mathbbm{1}$, where $\mathbbm{1}$ is the constant function at $1$. Further, $h_\statix(\reward + \gamma x)$ lies in this span for all $\reward \in \mathbb{R}$, and so the collection of functions $\{x \mapsto h_\statix(r + \gamma x) | r \in \mathbb{R} \}$ lies in a finite-dimensional subspace of functions. We may now appeal to Theorem 1 of \citet{engert1970finite}, which states that any finite-dimensional space of functions which is closed under translation is spanned by a set of functions of the form
    \begin{align}\label{eq:monomials}
        \bigcup_{j=1}^J \{ x \mapsto x^\ell \exp(\lambda_j x) \ |\  0 \leq \ell \leq L_j \} \, ,
    \end{align}
    for some finite subset $\{\lambda_1,\ldots,\lambda_J\}$ of $\mathbb{C}$. From this, we deduce that each function $x \mapsto h_\statix(x)$ may be expressed as a linear combination of functions of the form appearing in the set in expression \eqref{eq:monomials}. Further, enforcing the condition that the linear span must be closed under composition with $f_{r,\gamma}$ with $\gamma \in [0,1)$ rules out any values of $\lambda_j$ above which are not zero. Therefore, the linear span of the functions $h_1,\ldots,h_\numstats,\mathbbm{1}$ must be equal to the span of some set of monomials $x \mapsto x^\ell$, $0 \leq \ell \leq L$, for some $L \in \mathbb{N}$, and hence the statement of the theorem follows.
\end{proof}

\notClosed*

\begin{proof}
    (i) This follows as a special case of Theorem \ref{thm:classifyBellmanClosed}, since the statistics learnt by CDRL are expectations, as shown in Lemma \ref{lem:CDRLStats}.
    
    (ii) Quantiles cannot be expressed as expectations, and so we cannot appeal to Theorem \ref{thm:classifyBellmanClosed}. We instead proceed by describing a concrete counterexample to Bellman closedness. Fix a number $\numstats \in \mathbb{N}$ of quantiles. Consider an MDP with a single action, and an initial state $\state_0$ which transitions to one of two terminal states $\state_1$, $\state_2$ with equal probability. Suppose there is no immediate reward at state $\state_0$. We consider two different possibilities for reward distributions at states $\state_1$, $\state_2$, and show that these two possibilities yield the same quantiles for the return distributions at states $\state_1$ and $\state_2$, but different quantiles for the return distribution at state $\state_0$; thus demonstrating that finite sets of quantiles are not Bellman closed.
    
    Firstly, suppose rewards are drawn from $\mathrm{Unif}([0,1])$ at state $\state_1$ and $\mathrm{Unif}([1/\numstats, 1 + 1/\numstats])$ at $\state_2$, so that the $\frac{2\statix-1}{2\numstats}$-quantile of the return at states $\state_1$ and $\state_2$ are $\frac{2\statix-1}{2\numstats}$ and $\frac{2\statix+1}{2\numstats}$, for each $\statix=1,\ldots,\numstats$. Then the return distribution at state $\state_0$ is the mixture $\frac{1}{2}\mathrm{Unif}([0,\gamma]) + \frac{1}{2} \mathrm{Unif}([\gamma/\numstats, \gamma + \gamma/\numstats])$, and hence the $\frac{1}{2\numstats}$-quantile is $\frac{\gamma}{\numstats}$. Now, suppose instead that the reward distribution at state $\state_1$ is $\frac{1}{\numstats} \sum_{\statix=1}^\numstats \delta_{\frac{2\statix-1}{2\numstats}}$ and the reward distribution at $\state_2$ is $\frac{1}{\numstats} \sum_{\statix=1}^\numstats \delta_{\frac{2\statix+1}{2\numstats}}$. Then the $\frac{1}{2\numstats}$-quantile of the return distribution at state $\state_0$ is $\frac{3\gamma}{2\numstats}$.
\end{proof}

\subsection{Proofs of results from Section \ref{subsec:approximateBellmanClosed}}

In this section, we use operator notation reviewed in Section \ref{sec:olderalgorithms}. In both proofs, the supremum-Wasserstein distance will be of use, defined as $\overline{W}_1(\gendist_1, \gendist_2) = \sup_{(\state, \action) \in \statespace \times \actionspace} W_1(\gendist_1(\state, \action), \gendist_2(\state, \action))$ for all $\gendist_1, \gendist_2 \in \mathscr{P}_1(\mathbb{R})^{\statespace \times \actionspace}$. 
Before proving theorem \ref{thm:CDRLappBellmanClosed}, we state and prove an auxiliary lemma.

\begin{lemma}\label{lem:cdrlhelper}
    Let $\Pi_\mathcal{C}$ be the Cram\'er projection for equally-spaced support points $z_1 < \cdots < z_\numstats$, defined in Appendix Section \ref{sec:cdrl}.
    (i) $\Pi_\mathcal{C}$ is a non-expansion in $W_1$.
    (ii) For any distribution $\gendist \in \mathscr{P}(\mathbb{R})$ supported on $[z_1, z_\numstats]$, we have $W_1(\Pi_\mathcal{C} \gendist, \gendist) \leq \frac{z_K - z_1}{2(K-1)}$.
\end{lemma}
\begin{proof}
    In the proof of the first claim, we use the following characeterisation of the Cram\'er projection \citep{AnalysisCDRL}. For any distribution $\gendist \in \mathscr{P}(\mathbb{R})$ with CDF $F_\gendist$, the CDF of $\Pi_\mathcal{C} \mu$ is given by $F_{\Pi_\mathcal{C}\gendist}(v) = \frac{1}{z_{\statix+1} - z_{\statix}}\int_{z_\statix}^{z_{\statix+1}} F_\mu(t) \mathrm{d}t$ for $v \in [z_\statix, z_{\statix+1})$, $k=1,\ldots,\numstats-1$, with $F_{\Pi_\mathcal{C} \gendist}$ equal to $0$ on $(\infty, z_1)$ and equal to $1$ on $[z_\numstats, \infty)$.
    
    (i) Let $\gendist_1, \gendist_2 \in \mathscr{P}(\mathbb{R})$. We compute
    \begin{align*}
        W_1( \gendist_1, \gendist_2)
            \geq
        \sum_{\statix=1}^{\numstats-1} \int_{z_\statix}^{z_{\statix+1}} |F_{\gendist_1}(t) - F_{ \gendist_1}(t) | \mathrm{d}t
            \geq 
        \sum_{\statix=1}^\numstats (z_{\statix+1} - z_{\statix}) |F_{\Pi_\mathcal{C} \gendist_1}(z_\statix) - F_{\Pi_\mathcal{C} \gendist_1}(\statix) |
        = W_1(\Pi_\mathcal{C} \gendist_1, \Pi_\mathcal{C} \gendist_2) \, ,
    \end{align*}
    as required. The first inequality comes from expressing the Wasserstein distance between two distributions as the $L^1$ distance between their CDFs, and truncating the corresponding integral at $z_1$ and $z_\numstats$. The second inequality follows from Jensen's inequality.
    
    (ii) We first introduce some notation. Let $l,u : [z_1,z_\numstats] \rightarrow \{z_1,\ldots,z_\numstats\}$ be functions such that $l(y)$ is the largest element of $\{z_1,\ldots,z_\numstats\}$ which is less than or equal to $y$, and $u(y)$ is the smallest element of $\{z_1,\ldots,z_\numstats\}$ which is greater than or equal to $y$, for all $y \in [z_1,z_\numstats]$.
    A valid coupling between $\gendist$ and $\Pi_\mathcal{C}$ is then given as follows. Let $Y \sim \gendist$, and conditional on $Y$, let $p \sim \text{Bernoulli}\left(\frac{Y - l(Y)}{u(Y) - l(Y)}\right)$ if $Y \not\in \{z_1,\ldots,z_\numstats\}$, and $p=1$ almost surely conditional on $Y \in \{z_1,\ldots,z_\numstats\}$. Then define $Z = p l(Y) + (1-p) u(Y)$. It is straightforward to check that the marginal distribution of $Z$ is $\Pi_\mathcal{C} \gendist$, and we can straightforwardly upper-bound the transport cost associated with this coupling, by observing that for each possible value $y$ of $Y$, the contribution to the transport cost is $0$ if $y \in \{z_1,\ldots,z_\numstats\}$, and $\frac{u(y) - y}{u(y) - l(y)}(y-l(y)) + \frac{y- l(y)}{u(y) - l(y)}(u(y) - y) \leq \frac{u(y) - l(y)}{2} = \frac{z_\numstats - z_1}{2(\numstats-1)}$. Therefore, integrating over the distribution of $Y$ gives a transport cost of at most $\frac{z_\numstats - z_1}{2(\numstats-1)}$, which gives the required bound on the Wasserstein distance.
\end{proof}

\CDRLappBellmanClosed*

\begin{proof}
    For the CDRL statistics, we have $\statistic_{z_\statix, z_{\statix+1}}(\returndist{\policy}(\state, \action)) = \statistic_{z_\statix, z_{\statix+1}}(\Pi_\mathcal{C} \returndist{\policy}(\state, \action))$ for $\statix=1,\ldots,\numstats$ and all $(\state, \action) \in \statespace \times \actionspace$. Further, since $\Pi_\mathcal{C} \returndist{\policy}(\state, \action)$ is supported on $\{z_1,\ldots,z_\numstats\}$ for all $(\state, \action) \in \statespace \times \actionspace$, we have that $\statistic_{z_\statix, z_{\statix+1}}(\Pi_\mathcal{C} \returndist{\policy}(\state, \action)) = F^{-1}_{\Pi_\mathcal{C} \returndist{\policy}(\state, \action)}(z_{\statix})$. Let $(\approxdist(\state, \action) | (\state, \action) \in \statespace \times \actionspace)$ be the set of approximate distributions learnt by CDRL. As noted in Appendix Section \ref{sec:cdrl}, $\approxdist$ is the fixed point of the projected Bellman operator $\Pi_{\mathcal{C}}\BellmanOp{\policy}{}$, and $\returndist{\policy}$ is the fixed point of the Bellman operator $\BellmanOp{\policy}{}$.
    We now compute:
    \begin{align*}
        & \frac{1}{\numstats-1} \sum_{k=1}^{\numstats-1} \left|\statistic_{z_\statix, z_{\statix+1}}(\approxdist(\state, \action)) - \statistic_{z_\statix, z_{\statix+1}}(\returndist{\policy}(\state, \action))\right| \\
        = & 
        \frac{1}{\numstats-1} \sum_{k=1}^{\numstats-1} \left|\statistic_{z_\statix, z_{\statix+1}}(\approxdist(\state, \action)) - \statistic_{z_\statix, z_{\statix+1}}(\Pi_\mathcal{C} \returndist{\policy}(\state, \action)) \right| \\
        = & 
        \frac{1}{\numstats-1} \sum_{k=1}^{\numstats-1} \left|F^{-1}_{\approxdist(\state, \action)}(z_\statix) - F^{-1}_{\Pi_\mathcal{C} \returndist{\policy}(\state, \action)}(z_\statix) \right| \\
        = & \frac{1}{2\supportbound/(1-\gamma)}  \frac{2\supportbound/(1-\gamma)}{\numstats-1} \sum_{k=1}^{\numstats-1} \left|F^{-1}_{\approxdist(\state, \action)}(z_\statix) - F^{-1}_{\Pi_\mathcal{C} \returndist{\policy}(\state, \action)}(z_\statix) \right| \\
        = & \frac{1}{2\supportbound/(1-\gamma)} W_1(\eta, \Pi_\mathcal{C} \returndist{\policy}(\state, \action)) \\
        \overset{(a)}{=} & \frac{1}{2\supportbound/(1-\gamma)} W_1(\Pi_\mathcal{C}\BellmanOp{\policy}{}\approxdist, \Pi_\mathcal{C} \BellmanOp{\policy}{} \returndist{\policy}(\state, \action)) \\
        \overset{(b)}{\leq} & \frac{1}{2\supportbound/(1-\gamma)}\gamma \overline{W}_1(\approxdist, \returndist{\policy}) \\
        \overset{(c)}{\leq} & \frac{1}{2\supportbound/(1-\gamma)}\gamma \frac{1}{1 - \gamma} \overline{W}_1(\Pi_\mathcal{C}\returndist{\policy}, \returndist{\policy}) \\
        \overset{(d)}{\leq} &  \frac{1}{2\supportbound/(1-\gamma)}\gamma \frac{1}{1 - \gamma} \frac{\supportbound}{(1-\gamma)(\numstats-1)} \\
        = & \frac{\gamma}{2(1-\gamma)(\numstats-1)} \, ,
    \end{align*}
    as required. Here, (a) follows since $\approxdist$ is the fixed point of $\Pi_\mathcal{C} \BellmanOp{\policy}{}$ and $\returndist{\policy}$ is the fixed point of $\BellmanOp{\policy}{}$. (b) follows since $\Pi_\mathcal{C}$ is a non-expansion in $\overline{W}_1$, by Lemma~\ref{lem:cdrlhelper}.(i), and $\BellmanOp{\policy}{}$ is a $\gamma$-contraction in $\overline{W}_1$. (c) follows from the following argument:
    \begin{align*}
        \overline{W}_1(\approxdist, \returndist{\policy})
        \leq & \overline{W}_1(\approxdist, \Pi_\mathcal{C} \returndist{\policy}) + \overline{W}_1(\Pi_\mathcal{C} \returndist{\policy}, \returndist{\policy}) \\
        = & \overline{W}_1(\Pi_\mathcal{C} \BellmanOp{\policy}{}\approxdist, \Pi_\mathcal{C} \BellmanOp{\policy}{} \returndist{\policy}) + \overline{W}_1(\Pi_\mathcal{C} \returndist{\policy}, \returndist{\policy}) \\
        \leq & \gamma \overline{W}_1(\approxdist, \returndist{\policy}) + \overline{W}_1(\Pi_\mathcal{C} \returndist{\policy}, \returndist{\policy}) \\
        \implies \overline{W}_1(\approxdist, \returndist{\policy}) \leq & \frac{1}{1 - \gamma} \overline{W}_1(\Pi_\mathcal{C} \returndist{\policy}, \returndist{\policy}) \, .
    \end{align*}
    Finally, (d) follows from Lemma~\ref{lem:cdrlhelper}.(ii).
\end{proof}

Before giving a proof of Theorem \ref{thm:QDRLappBellmanClosed}, we first state and prove a lemma that will be useful.

\begin{lemma}\label{lem:qdrlhelper}
Let $\tau_\statix = \frac{2\statix-1}{2\numstats}$ for $\statix=1,\ldots,\numstats$, and consider the corresponding Wasserstein-1 projection operator $\Pi_{W_1} : \mathscr{P}(\mathbb{R}) \rightarrow \mathscr{P}(\mathbb{R})$, defined by
\begin{align*}
    \Pi_{W_1}( \mu ) = \frac{1}{\numstats} \sum_{\statix=1}^\numstats \delta_{F^{-1}_\mu(\tau_k)} \, ,
\end{align*}
for all $\mu \in \mathscr{P}(\mathbb{R})$, where $F^{-1}_{\mu}$ is the inverse c.d.f. of $\mu$.
Let $\eta_1, \eta_2 \in \mathscr{P}(\mathbb{R})$, such that $\sup(\support(\eta_i)) - \inf(\support(\eta_i)) \leq \supportwidth$ for $i=1,2$. Then we have:
\begin{align*}
&(i) \; W_1(\Pi_{W_1}\eta_1, \eta_1) \leq \frac{\supportwidth}{\numstats} \, ;\\
&(ii) \; W_1(\Pi_{W_1}\eta_1, \Pi_{W_1}\eta_2) \leq W_1(\eta_1, \eta_2) + \frac{2 \supportwidth}{\numstats} \, .
\end{align*}
\end{lemma}

\begin{proof}
We start by proving (i). Let $F^{-1}_{\eta_1}$ be the inverse c.d.f  of $\eta_1$. We have
\begin{align*}
W_1(\mu, \Pi_{W_1}\mu) & = \sum^{\numstats-1}_{i=0} \frac{1}{\numstats} \expected_{X \sim \mu}\left\lbrack \left|X - F^{-1}_{\eta_1}\left(\frac{2i+1}{2\numstats}\right) \right| \Bigg| F^{-1}_{\eta_1}\left(\frac{i}{\numstats}\right) \leq X \leq F^{-1}_{\eta_1}\left(\frac{i+1}{\numstats}\right)\right\rbrack \\
    & \leq \frac{1}{\numstats}\left(F^{-1}_{\eta_1}(1) - F^{-1}_{\eta_1}(0)\right) \\
    & = \frac{\supportwidth}{\numstats}
\end{align*}
We can now prove (ii), using the triangle inequality and (i):
\begin{align*}
W_1(\Pi_{W_1}\eta_1, \Pi_{W_1}\eta_2) &\leq W_1(\Pi_{W_1}\eta_1, \eta_1) + W_1(\eta_1, \eta_2) + W_1(\eta_2, \Pi_{W_1}\eta_2)\\ 
&\leq W_1(\eta_1, \eta_2) + \frac{2\supportwidth}{\numstats} \, .
\end{align*}
\end{proof}

\QDRLappBellmanClosed*

\begin{proof}
Let $(\approxstatistic_{1:\numstats}(\state, \action) | (\state, \action) \in \statespace \times \actionspace)$ be the collection of statistics learnt under QDRL. We denote by $\approxdist(\state, \action)$ the distribution imputed from the statistics $\approxstatistic_{1:\numstats}(\state, \action)$, for each $(\state, \action) \in \statespace \times \actionspace$. As noted in Appendix Section \ref{sec:qdrl}, $\approxdist$ is the fixed point of the projected Bellman operator $\Pi_{W_1}\BellmanOp{\policy}{}$, and $\eta_\pi$ is the fixed point of $\BellmanOp{\policy}{}$.
We begin by noting that if all immediate reward distributions have support contained within $[-\supportbound, \supportbound]$, then the true and learnt reward distributions are supported on $[-\supportbound/(1-\gamma), \supportbound/(1-\gamma)]$, and further, so are the distributions $\BellmanOp{\policy}{}\approxdist(\state, \action)$ for each $(\state, \action) \in \statespace\times\actionspace$.
We thus compute
\begin{align*}
    & \sup_{(\state, \action) \in \statespace\times\actionspace} \frac{1}{\numstats} \sum_{\statix=1}^\numstats |\statistic_\statix(\returndist{\policy}(\state, \action)) - \approxstatistic_\statix(\state, \action)| \\
    = & \sup_{(\state, \action) \in \statespace\times\actionspace} W_1(\Pi_{W_1}\approxdist(\state, \action), \Pi_{W_1}\returndist{\policy}(\state, \action)) \\
    \leq & \sup_{(\state, \action) \in \statespace\times\actionspace} W_1(\approxdist(\state, \action), \returndist{\policy}(\state, \action)) + \frac{4 \supportbound}{K(1 - \gamma)} \, ,
\end{align*}
with the inequality following from Lemma \ref{lem:qdrlhelper}(ii). From here, we note that
\begin{align*}
    \sup_{(\state, \action) \in \statespace\times\actionspace} W_1(\approxdist(\state, \action), \returndist{\policy}(\state, \action)) &
    \overset{(a)}\leq \sup_{(\state, \action) \in \statespace\times\actionspace} \left\lbrack W_1(\approxdist(\state, \action), \Pi_{W_1} \returndist{\policy}(\state, \action))  + W_1(\Pi_{W_1} \returndist{\policy}(\state, \action), \returndist{\policy}(\state, \action)) \right\rbrack \\
    & \overset{(b)}{\leq} \sup_{(\state, \action) \in \statespace\times\actionspace} W_1(\approxdist(\state, \action), \Pi_{W_1} \returndist{\policy}(\state, \action))  + \frac{2\supportbound}{\numstats(1 - \gamma)} \\
    & \overset{(c)}{=} \sup_{(\state, \action) \in \statespace\times\actionspace} W_1(\Pi_{W_1} \BellmanOp{\policy}{}\approxdist(\state, \action), \Pi_{W_1} \BellmanOp{\policy}{} \returndist{\policy}(\state, \action)) + \frac{2\supportbound}{\numstats(1 - \gamma)} \\
    & \overset{(d)}{\leq} \sup_{(\state, \action) \in \statespace\times\actionspace} W_1( \BellmanOp{\policy}{}\approxdist(\state, \action), \BellmanOp{\policy}{} \returndist{\policy}(\state, \action)) + \frac{4\supportbound}{\numstats(1 - \gamma)}  + \frac{2\supportbound}{\numstats(1 - \gamma)} \\
    & \overset{(e)}\leq \sup_{(\state, \action) \in \statespace\times\actionspace} \gamma W_1( \approxdist(\state, \action), \returndist{\policy}(\state, \action)) + \frac{6\supportbound}{\numstats(1 - \gamma)} \\
    \implies \sup_{(\state, \action) \in \statespace\times\actionspace} W_1(\approxdist(\state, \action), \returndist{\policy}(\state, \action)) & \leq \frac{6\supportbound}{K(1 - \gamma)^2} \, .
\end{align*}
Here, (a) follows from the triangle inequality, (b) follows from Lemma \ref{lem:qdrlhelper}(i). (c) follows since $\approxdist$ is the fixed point of $\Pi_{W_1} \BellmanOp{\policy}{}$ and $\returndist{\policy}$ is the fixed point of $\BellmanOp{\policy}{}$. (d) follows from Lemma \ref{lem:qdrlhelper}(ii), where we use the fact that the support of the distributions constituting the fixed points of $\Pi_{W_1}\BellmanOp{\policy}{}$ and $\BellmanOp{\policy}{}$ necessarily are supported on $[-\supportbound/(1-\gamma), \supportbound/(1-\gamma)]$.
(e) follows from the $\gamma$-contractivity of the Bellman operator $\BellmanOp{\policy}{}$ with respect to the metric $\sup_{(\state, \action) \in \statespace \times \actionspace} W_1(\mu_1(\state, \action), \mu_2(\state, \action))$, for $\mu_1, \mu_2 \in \mathscr{P}(\mathbb{R})^{\statespace \times \actionspace}$ \citep{C51}.
Hence, we obtain
\begin{align*}
    \sup_{(\state, \action) \in \statespace\times\actionspace} \frac{1}{\numstats} \sum_{\statix=1}^\numstats |\statistic_\statix(\returndist{\policy}(\state, \action)) - \approxstatistic_\statix(\state, \action)| & \leq \frac{6\supportbound}{K(1 - \gamma)^2} + \frac{4 \supportbound}{K(1 - \gamma)} \\
    & = \frac{2\supportbound(5 - 2\gamma)}{K(1-\gamma)^2} \, .
\end{align*}
\end{proof}

\subsection{Proofs of results from Section \ref{sec:meanConsistency}}

\CDRLMeanQDRLNoMean*

\begin{proof}
    (i) The statistics learnt by CDRL are of the form $\statistic_\statix(\gendist) = \mathbb{E}_{\genrv \sim \gendist}\left\lbrack h_{z_\statix, z_{\statix+1}}(Z) \right\rbrack$, for $k=1,\ldots,\numstats-1$. We observe that the mean functional $m(\gendist) = \mathbb{E}_{\genrv \sim \gendist}\left\lbrack Z \right\rbrack$ is contained in the linear span of $s_{0:\numstats-1}$, where $s_0(\gendist) = 1$ for all $\gendist$. Indeed,
    \begin{align*}
        m = R_{\mathrm{max}} s_0 - \left( \frac{R_{\mathrm{max}} - R_{\mathrm{min}}}{K} \right) \sum_{k=1}^{\numstats-1}\statistic_\statix \, ,
    \end{align*}
    since
    \begin{align*}
        x = R_{\mathrm{max}} - \left( \frac{R_{\mathrm{max}} - R_{\mathrm{min}}}{K} \right) \sum_{k=1}^{\numstats-1}h_{z_{\statix},z_{\statix+1}}(x)
    \end{align*}
    for all $x \in [-R_\mathrm{min}, R_\mathrm{max}]$. Since the singleton set consisting of the mean functional is Bellman closed, it follows that whatever distribution is imputed, the effective update to the mean of the distribution at the current state is the same as updating according to the classical Bellman update for the mean.
    
    (ii) We note that the mean is not encoded by a finite set of quantiles, and hence it is impossible for expected returns to be correctly in general. To make this concrete, fix a number $K$ of quantiles to be learnt, and consider a single state, two action MDP, with reward distribution $\frac{4K-1}{4K}\delta_0 + \frac{1}{4K}\delta_1$ for the first action, and reward distribution $\delta_{1/8K}$ for the second action. Fitting quantiles at $\tau\in\{ \frac{2k-1}{2K} | k=1,\ldots,K \}$ results in all quantiles for the first distribution being equal to $0$, and thus the imputed distribution is $\delta_0$, resulting in a imputed mean of $0$. By constrast, for the second distribution, all quantiles are fitted at $1/8K$, resulting in an imputed distribution of $\delta_{1/8K}$ and an imputed mean of $1/8K$. Thus, a QDRL control algorithm will act greedily with respect to these imputed means and select the second action, which is sub-optimal as the first action has higher expected reward.
\end{proof}

%% file: appendixAdditionalResults.tex
\section{Additional theoretical results}\label{sec:additionalTheory}

In this section, we provide several examples to illustrate the point made in Section \ref{subsec:approximateBellmanClosed} that in general, it is not possible to simultaneously achieve low approximation error on all statistics in a non-Bellman closed collection.

\begin{restatable}{lemma}{noUniformStatApprox}
For a fixed $\numstats \in \mathbb{N}$, let $\statistic_{1:\numstats-1}$ be the statistics corresponding to CDRL (with fixed discount factor $\gamma \in [0, 1)$) with equally spaced support $R_\text{min} = z_1 < \ldots < z_\numstats = R_\text{max}$. As earlier in the paper, we denote by $\approxstatistic_\statix(\state, \action)$ the relevant \emph{learnt} value of the statistic concerned. Then we have:
\begin{align*}
    \sup_{\substack{\mathcal{M} \text{ MDP} \\ \pi \text{ policy} }} \sup_{\substack{\state \in \statespace \\ \action \in \actionspace}}\sup_{\statix=1,\ldots,\numstats-1} | \approxstatistic_\statix(\state, \action) - \statistic_\statix(\returndist{\policy}(\state, \action)) | \not\rightarrow 0
\end{align*}
as $\numstats \rightarrow \infty$.
\end{restatable}

\begin{proof}
We work with a particular family of MDPs with two states $\state_1, \state_2$, one action in each state, with $\state_1$ transitioning to $\state_2$ with probability $1$, and $\state_2$ terminal. In such MDPs, there is only one policy, which we denote by $\policy$; and we drop notational dependence on actions for clarity. No rewards are received at state $\state_1$; we specify the rewards received at state $\state_2$ below. We take a discount factor $\gamma = \frac{2^{\lemmaindex}}{2^\lemmaindex + 1}$ for some $k \in \mathbb{N}$. Fix $\cdrlDepthIx \in \mathbb{N}$, and consider CDRL updates with bin locations at $z_\statix = \frac{\statix}{2^\cdrlDepthIx}$ for $\statix=0,\ldots,2^\cdrlDepthIx$. Specifically, consider learning the statistic
\begin{align*}
    \mathbb{E}_{\genrv \sim \returndist{\policy}(\state_1)}\left\lbrack h_{\frac{1}{2}, \frac{1}{2}+\frac{1}{2^\cdrlDepthIx}}(\genrv)\right\rbrack \, .
\end{align*}
Since there are no rewards recieved at state $\state_1$, at convergence the estimate of this statistic (which we denote by $\hat{s}(x_1)$) is equal to
\begin{align*}
    \mathbb{E}_{\genrv \sim \hat{\returndist{}}(\state_2)}\left\lbrack h_{\frac{1}{2}, \frac{1}{2}+\frac{1}{2^\cdrlDepthIx}}(\gamma\genrv)\right\rbrack = 
    \mathbb{E}_{\genrv \sim \hat{\returndist{}}(\state_2)}\left\lbrack h_{\frac{\gamma^{-1}}{2}, \frac{\gamma^{-1}}{2}+\frac{\gamma^{-1}}{2^\cdrlDepthIx}}(\genrv)\right\rbrack =
    \mathbb{E}_{\genrv \sim \hat{\returndist{}}(\state_2)}\left\lbrack h_{\frac{1}{2}+\frac{1}{2^{\lemmaindex+1}}, \frac{1}{2}+\frac{1}{2^{\lemmaindex+1}}+\frac{1}{2^\cdrlDepthIx}+\frac{1}{2^{\cdrlDepthIx+\lemmaindex}}}(\genrv)\right\rbrack
\end{align*}
where $\hat{\returndist{}}(\state_2)$ is the approximate return distribution learnt at state $\state_2$. 
Now, consider two possible reward distributions at state $\state_2$:
\begin{align*}
    \rho_A = \delta_{\frac{1}{2} + \frac{1}{2^{\lemmaindex+1}} + \frac{3}{2^{\cdrlDepthIx+1}}  } \, , \text{ and } \rho_B = \frac{1}{2}\left( \delta_{\frac{1}{2} + \frac{1}{2^{\lemmaindex+1}} + \frac{1}{2^{\cdrlDepthIx}}} + \delta_{\frac{1}{2} + \frac{1}{2^{\lemmaindex+1}} + \frac{2}{2^{\cdrlDepthIx}}} \right) \, .
\end{align*}
Under these two reward distributions, the fitted distribution $\approxdist(\state_2)$ is the same, namely $\rho_B$, and thus the estimate $\hat{\statistic}(\state_1)$  is the same. Our aim is to show that for these two different reward distributions, the difference of the true values of the statistic $\hat{\statistic}(\state_1)$ is independent of $\cdrlDepthIx$, and hence the value of $\hat{\statistic}(\state_1)$ cannot converge to the true statistic as $\cdrlDepthIx \rightarrow \infty$. To achieve this, and finish the proof, we calculate directly. In the case where the reward distribution at state $\state_2$ is $\rho_A$, we have (assuming $\cdrlDepthIx > \lemmaindex + 1$)
\begin{align*}
    \statistic(\returndist{\policy}(\state_1)) = \mathbb{E}_{\genrv \sim \rho_A}\left\lbrack h_{\frac{1}{2}+\frac{1}{2^{\lemmaindex+1}}, \frac{1}{2}+\frac{1}{2^{\lemmaindex+1}}+\frac{1}{2^\cdrlDepthIx}+\frac{1}{2^{\cdrlDepthIx+\lemmaindex}}}(\genrv)\right\rbrack = 0 \, .
\end{align*}
In the case where the reward distribution at state $\state_2$ is $\rho_B$, we have
\begin{align*}
    \statistic(\returndist{\policy}(\state_1)) & = \mathbb{E}_{\genrv \sim \rho_B}\left\lbrack h_{\frac{1}{2}+\frac{1}{2^{\lemmaindex+1}}, \frac{1}{2}+\frac{1}{2^{\lemmaindex+1}}+\frac{1}{2^\cdrlDepthIx}+\frac{1}{2^{\cdrlDepthIx+\lemmaindex}}}(\genrv)\right\rbrack =  \\
    & = \frac{1}{2} \left( \frac{ (\frac{1}{2} + \frac{1}{2^{\lemmaindex+1}} + \frac{1}{2^{\cdrlDepthIx}})  - (\frac{1}{2}+\frac{1}{2^{\lemmaindex+1}}+\frac{1}{2^\cdrlDepthIx}+\frac{1}{2^{\cdrlDepthIx+\lemmaindex}})  }{(\frac{1}{2}+\frac{1}{2^{\lemmaindex+1}}) - (\frac{1}{2}+\frac{1}{2^{\lemmaindex+1}}+\frac{1}{2^\cdrlDepthIx}+\frac{1}{2^{\cdrlDepthIx+\lemmaindex}})} \right) \\
    & = \frac{1}{2}\left( \frac{1}{2^\lemmaindex + 1} \right) \, .
\end{align*}
\end{proof}

\begin{restatable}{lemma}{noUniformStatApproxQuantile}
For a fixed $\numstats \in \mathbb{N}$, let $\statistic_{1:\numstats-1}$ be the statistical functionals corresponding to by QDRL (with fixed discount factor $\gamma \in [0, 1)$). As earlier in the paper, we denote by $\approxstatistic_\statix(\state, \action)$ the relevant \emph{learnt} value of the statistic concerned. Then we have:
\begin{align*}
    \sup_{\substack{\mathcal{M} \text{ MDP} \\ \pi \text{ policy} }} \sup_{\substack{\state \in \statespace \\ \action \in \actionspace}} \sup_{\statix=1,\ldots,\numstats} | \approxstatistic_\statix(\state, \action) - s_\statix(\returndist{\policy}(\state, \action)) | \not\rightarrow 0
\end{align*}
as $\numstats \rightarrow \infty$.
\end{restatable}

\begin{proof}
We work with a particular family of MDPs with three states $\state_0, \state_1, \state_2$, one action in each state, with $\state_0$ transitioning to $\state_1$ with probability $\frac{1}{2} - \varepsilon$ and with $\state_0$ transitioning to $\state_2$ with probability $\frac{1}{2} + \varepsilon$ with $\varepsilon \ll 1$. We take $\state_1$ and $\state_2$ to be terminal, no rewards are received at state $\state_0$; we specify the rewards received at state $\state_1$ and $\state_2$ below. We suppose in the following that $\numstats$ is odd.

The reward distributions at state $\state_1$ and $\state_2$ are given by
\begin{align*}
    \rho_1 = \left(\frac{1}{2\numstats} - \varepsilon\right)\delta_{0} + \left(\frac{2\numstats - 1}{2\numstats} + \varepsilon\right)\delta_{1} \, , \text{ and } \rho_2 = \left(\frac{1}{2\numstats} - \varepsilon\right)\delta_{0} + \left(\frac{2\numstats - 1}{2\numstats} + \varepsilon\right)\delta_{-1} \, .
\end{align*}
Under these reward distributions the fitted return distributions are:
\begin{align*}
    \approxdist(\state_1) = \delta_{1}\, , \text{ and } \approxdist(\state_2) = \delta_{-1}\,\, .
\end{align*}
Therefore, we have
\begin{align*}
\statistic_{\frac{\numstats + 1}{2}}(\returndist{\policy}(\state_0)) = 0\, , \text{ and } \approxstatistic_{\frac{\numstats + 1}{2}}(\state_0) = -\gamma\, .
\end{align*}
\end{proof}

%% file: appendixExperiments.tex
\section{ER-DQN experimental details}\label{sec:erdqndetails}

\subsection{ER-DQN architecture}\label{sec:ERDQNArchitecture}
    
As discussed in Section \ref{sec:atari}, the ER-DQN architecture matches the exact architecture of QR-DQN \citep{QRDQN}. The Q-network, for a given input $x$, outputs expectiles $e_{\tau_{1:\numstats}}(\state, \action)$ for each $\action \in \actionspace$. In our experiments with 11 statistics, we take $\tau_{1:11}$ to be linearly spaced with $\tau_1 = 0.01$, $\tau_{11}=0.99$. Note that we have $\tau_6 = 0.5$, and thus this expectile statistic is in fact the mean. For the purposes of control, greedy actions at a state $\state \in \statespace$ are thus selected according to $\argmax_{\action \in \actionspace} e_{\tau_6}(\state, \action)$, rather than averaging over statistics as in QR-DQN. For the imputation strategy, we take the root-finding problem in Expression \eqref{eq:rootFindingProblem}, and use a call to the SciPy \texttt{root} routine with default parameters.

\subsection{Training details}\label{sec:erdqntraining}

We use the Adam optimiser with a learning rate of 0.00005, after testing learning rates 0.00001, 0.00003, 0.00005, 0.00007, and 0.0001 on a subset of 6 Atari games. All other hyperparameters in training correspond to those used in \cite{QRDQN}. In particular, the target distribution is computed from a target network. Note that each training pass requires a call to the SciPy optimiser to compute the imputed samples, and thus in general will be more computationally expensive than other deep distributional Q-learning-style agents, such as C51 and QR-DQN. However, by parallelising the optimiser calls for a minibatch of transitions across several CPUs, we found that training times when using 11 expectiles to be comparable to training times of QR-DQN.

For ER-DQN Naive, we found that results were slightly improved by using 201 expectiles compared to 11, so include results with this larger number of statistics in the main paper. We take $\tau_{1:201}$ according to the same prescription as for QR-DQN: linearly spaced, with $\tau_1 = 1/(2\times201)$ and $\tau_{201} = 1 - 1/(2\times201)$.

\subsection{Environment details}
We use the Arcade Learning Environment \citep{bellemare2013arcade} to train and evaluate ER-DQN on a selection of 57 Atari games. The precise parameter settings of the environment are exactly the same as in the experiments performed on QR-DQN, to allow for direct comparison.

\subsection{Detailed results}

In addition to the human normalised mean/median results presented in the main paper, we include training curves for all 4 evaluated agents on all 57 Atari games in Figure~\ref{fig:atari_all}, and raw maximum scores attained in Table~\ref{fig:atari_individ_results}.

\begin{figure}
    \centering
    \includegraphics[keepaspectratio,width=\textwidth]{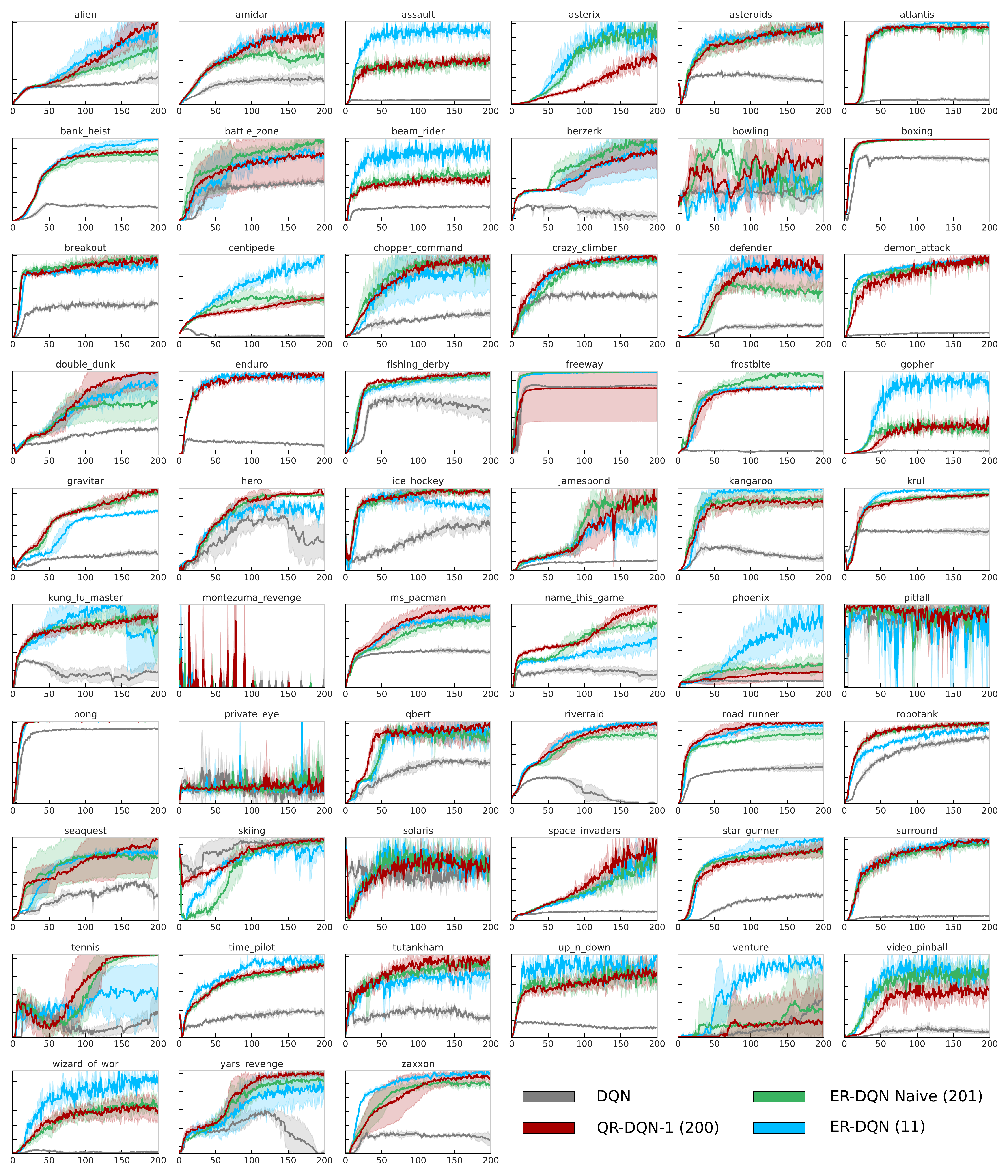}
    \caption{Training curves for DQN, QR-DQN-1, ER-DQN Naive, and ER-DQN on all 57 Atari games.}
    \label{fig:atari_all}
\end{figure}

\begin{table}
\centering{
\small
\begin{tabular}{ l | r|r|r}
\textbf{\textsc{games}} & \textbf{\textsc{QR-DQN-1}} & \textbf{\textsc{ER-DQN Naive}} & \textbf{\textsc{ER-DQN}} \\
\hline
alien & \textbf{\textcolor{blue}{7279.5}} & 5056.2 & 6212.0 \\
amidar & 2235.8 & 1528.8 & \textbf{\textcolor{blue}{2313.0}} \\
assault & 17653.9 & 19156.2 & \textbf{\textcolor{blue}{25826.8}} \\
asterix & 306055.9 & 366152.1 & \textbf{\textcolor{blue}{434743.6}} \\
asteroids & 3484.4 & 3250.9 & \textbf{\textcolor{blue}{3793.2}} \\
atlantis & 947995.0 & 939050.0 & \textbf{\textcolor{blue}{974408.3}} \\
bank\_heist & 1185.7 & 1132.5 & \textbf{\textcolor{blue}{1326.5}} \\
battle\_zone & 33987.2 & \textbf{\textcolor{blue}{40805.3}} & 35098.5 \\
beam\_rider & 25095.7 & 29542.5 & \textbf{\textcolor{blue}{48230.1}} \\
berzerk & 2151.2 & 2626.6 & \textbf{\textcolor{blue}{2749.8}} \\
bowling & 58.0 & \textbf{\textcolor{blue}{63.4}} & 53.1 \\
boxing & 99.5 & 99.4 & \textbf{\textcolor{blue}{99.9}} \\
breakout & 505.2 & \textbf{\textcolor{blue}{538.6}} & 509.8 \\
centipede & 11465.1 & 12325.3 & \textbf{\textcolor{blue}{22505.9}} \\
chopper\_command & \textbf{\textcolor{blue}{12767.2}} & 11765.8 & 11886.1 \\
crazy\_climber & 159244.2 & 158369.9 & \textbf{\textcolor{blue}{161040.2}} \\
defender & \textbf{\textcolor{blue}{41098.7}} & 32225.2 & 36473.5 \\
demon\_attack & \textbf{\textcolor{blue}{114530.2}} & 108496.2 & 111921.2 \\
double\_dunk & \textbf{\textcolor{blue}{16.5}} & 4.0 & 16.3 \\
enduro & 2294.1 & 1923.9 & \textbf{\textcolor{blue}{2339.5}} \\
fishing\_derby & \textbf{\textcolor{blue}{21.6}} & 18.4 & 20.2 \\
freeway & 27.2 & \textbf{\textcolor{blue}{34.0}} & 33.9 \\
frostbite & 4068.1 & \textbf{\textcolor{blue}{5408.0}} & 4233.7 \\
gopher & 82060.6 & 86874.1 & \textbf{\textcolor{blue}{115828.3}} \\
gravitar & 937.0 & \textbf{\textcolor{blue}{942.8}} & 680.9 \\
hero & \textbf{\textcolor{blue}{23799.1}} & 21916.6 & 20374.5 \\
ice\_hockey & \textbf{\textcolor{blue}{-1.7}} & -1.9 & -2.7 \\
jamesbond & 5298.5 & \textbf{\textcolor{blue}{5440.4}} & 4113.6 \\
kangaroo & 14827.6 & 15371.1 & \textbf{\textcolor{blue}{15954.4}} \\
krull & 10591.2 & 10738.0 & \textbf{\textcolor{blue}{11318.5}} \\
kung\_fu\_master & 49695.5 & 52080.6 & \textbf{\textcolor{blue}{58802.2}} \\
montezuma\_revenge & \textbf{\textcolor{blue}{0.1}} & 0.0 & 0.0 \\
ms\_pacman & \textbf{\textcolor{blue}{5860.4}} & 4856.1 & 5048.5 \\
name\_this\_game & \textbf{\textcolor{blue}{20509.1}} & 17064.9 & 13090.9 \\
phoenix & 15475.2 & 25177.3 & \textbf{\textcolor{blue}{91189.4}} \\
pitfall & \textbf{\textcolor{blue}{0.0}} & \textbf{\textcolor{blue}{0.0}} & \textbf{\textcolor{blue}{0.0}} \\
pong & 21.0 & \textbf{\textcolor{blue}{21.0}} & 21.0 \\
private\_eye & \textbf{\textcolor{blue}{531.3}} & 388.3 & 176.3 \\
qbert & \textbf{\textcolor{blue}{17573.5}} & 14536.0 & 17418.4 \\
riverraid & 18125.3 & 15726.4 & \textbf{\textcolor{blue}{18472.2}} \\
road\_runner & \textbf{\textcolor{blue}{67084.8}} & 57168.0 & 64577.7 \\
robotank & \textbf{\textcolor{blue}{58.0}} & 56.7 & 54.8 \\
seaquest & 16143.3 & 13501.0 & \textbf{\textcolor{blue}{19401.0}} \\
skiing & -16869.1 & -15085.4 & \textbf{\textcolor{blue}{-10528.6}} \\
solaris & 2615.3 & 2483.3 & \textbf{\textcolor{blue}{2810.6}} \\
space\_invaders & 11873.3 & 10099.6 & \textbf{\textcolor{blue}{14265.7}} \\
star\_gunner & 76556.3 & 75404.8 & \textbf{\textcolor{blue}{88900.3}} \\
surround & 8.4 & 8.2 & \textbf{\textcolor{blue}{8.6}} \\
tennis & \textbf{\textcolor{blue}{22.8}} & 22.7 & 5.8 \\
time\_pilot & 9902.0 & 10009.6 & \textbf{\textcolor{blue}{11675.5}} \\
tutankham & \textbf{\textcolor{blue}{282.8}} & 256.7 & 237.9 \\
up\_n\_down & \textbf{\textcolor{blue}{44893.6}} & 35169.7 & 32083.3 \\
venture & 266.5 & 476.7 & \textbf{\textcolor{blue}{1107.0}} \\
video\_pinball & 570852.7 & 603852.1 & \textbf{\textcolor{blue}{727091.1}} \\
wizard\_of\_wor & 21667.1 & 24397.5 & \textbf{\textcolor{blue}{36049.8}} \\
yars\_revenge & \textbf{\textcolor{blue}{27264.3}} & 26056.7 & 24099.4 \\
zaxxon & 11707.1 & 11120.2 & \textbf{\textcolor{blue}{12264.4}} \\
\end{tabular}
}
\caption{Raw max test scores across all 57 Atari games, starting with 30 no-op actions.}
\label{fig:atari_individ_results}
\end{table}

%% file: appendixExamples.tex
\section{Additional experimental results}\label{sec:examples}

In Section \ref{sec:edrlTabular}, we saw that the expectiles learned by EDRL-Naive on an $N$-Chain with length 15 collapsed, whereas the expectiles learnt by EDRL were reasonable approximations to the true expectiles of the return distribution. This resulted in lower average expectile estimation error with the latter expectiles, as described in Definition \ref{def:approxbc}. In Figure \ref{fig:nchainEDRLCombined}, we supplement this by plotting Wasserstein distance between an imputed distribution for the learnt statistics and the true return distribution. This gives an alternate metric which additionally indicates how well the collection of learnt statistics summarises the full return distribution. Under this metric, we observe that increasing the number of expectiles always leads to improved performance under EDRL, whilst for EDRL-Naive, poor Wasserstein reconstruction error is observed for large numbers of expectiles and/or distance from the goal state.

We also include results for $N$-chain environments with different reward distributions, observing qualitatively similar phenomena as those noted for Figure \ref{fig:nchainEDRLCombined}. Specifically, we use two additional variants of the reward distribution at the goal state: uniform and Gaussian. We plot average expectile error in Figure \ref{fig:unifGaussianExpectiles}, and Wasserstein distance between imputed and true return distributions in Figure \ref{fig:unifGaussianWasserstein}.

\begin{figure}[h]
    \centering
    \null
    \hfill
     \subfloat[Expectile estimation error\label{fig:nchainEDRLCombined1}]{%
       \includegraphics[width=0.38\textwidth]{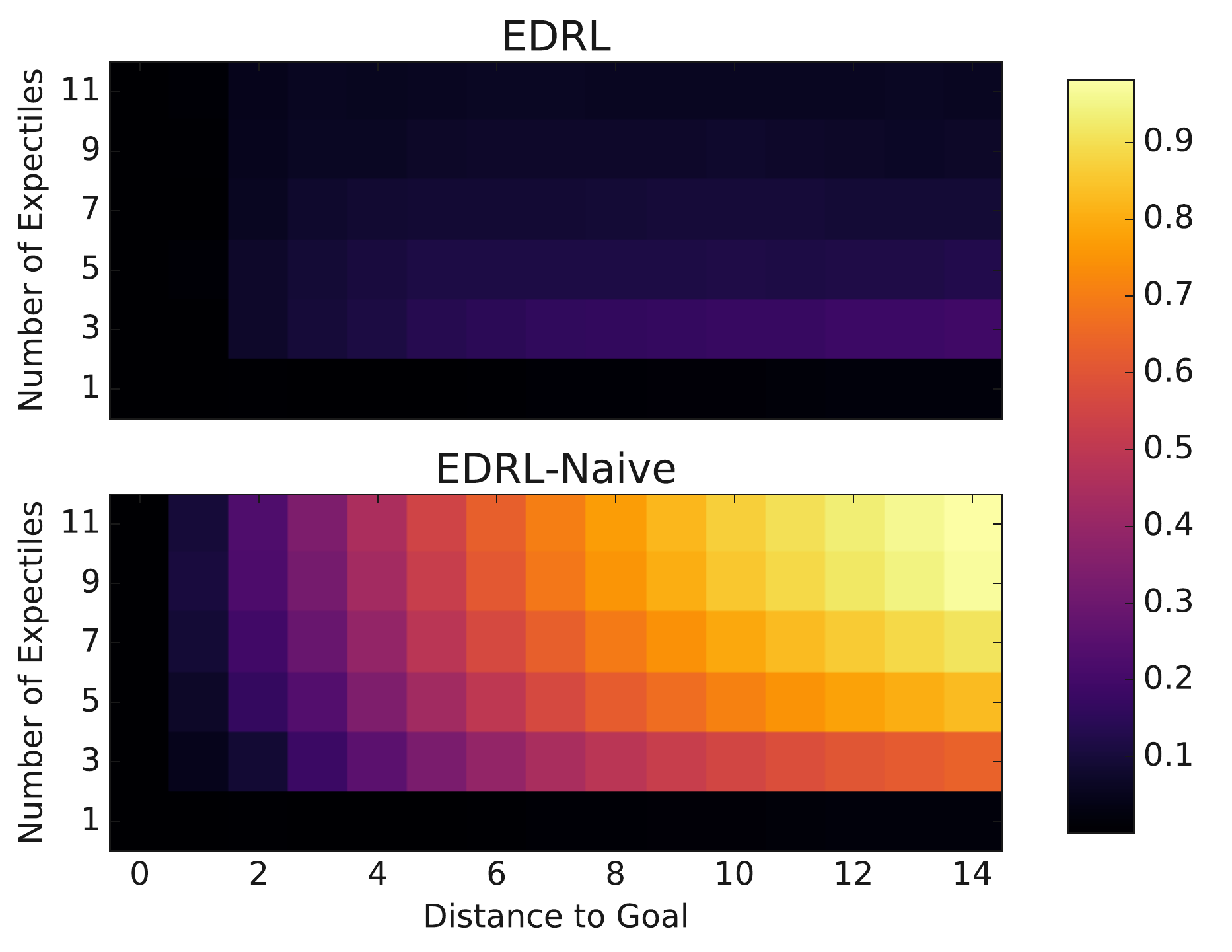}
     }
     \hfill
     \subfloat[1-Wasserstein distance on imputed samples\label{fig:nchainEDRLCombined2}]{%
       \includegraphics[width=0.38\textwidth]{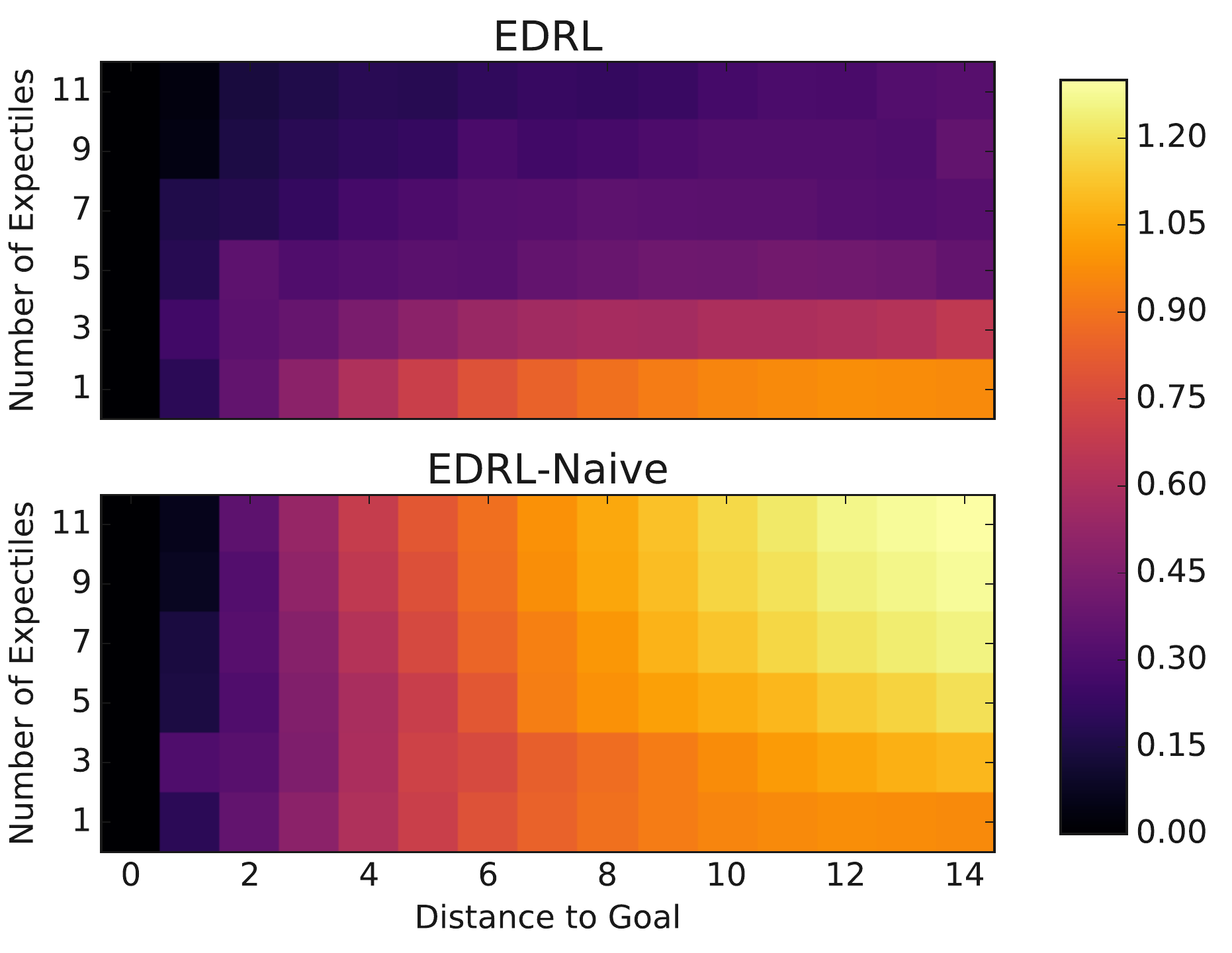}
     }
     \hfill
     \null
     \caption{Expectile estimation error and 1-Wasserstein distance between imputed samples and the true return distribution for varying numbers of learned expectiles and different $N$-Chain lengths.}
     \label{fig:nchainEDRLCombined}
\end{figure}

\begin{figure}[h]
    \centering
    \null
    \hfill
     \subfloat[Uniform (-1, 1) reward distribution.\label{fig:nchainExpectileErrorUniform}]{%
       \includegraphics[width=0.38\textwidth]{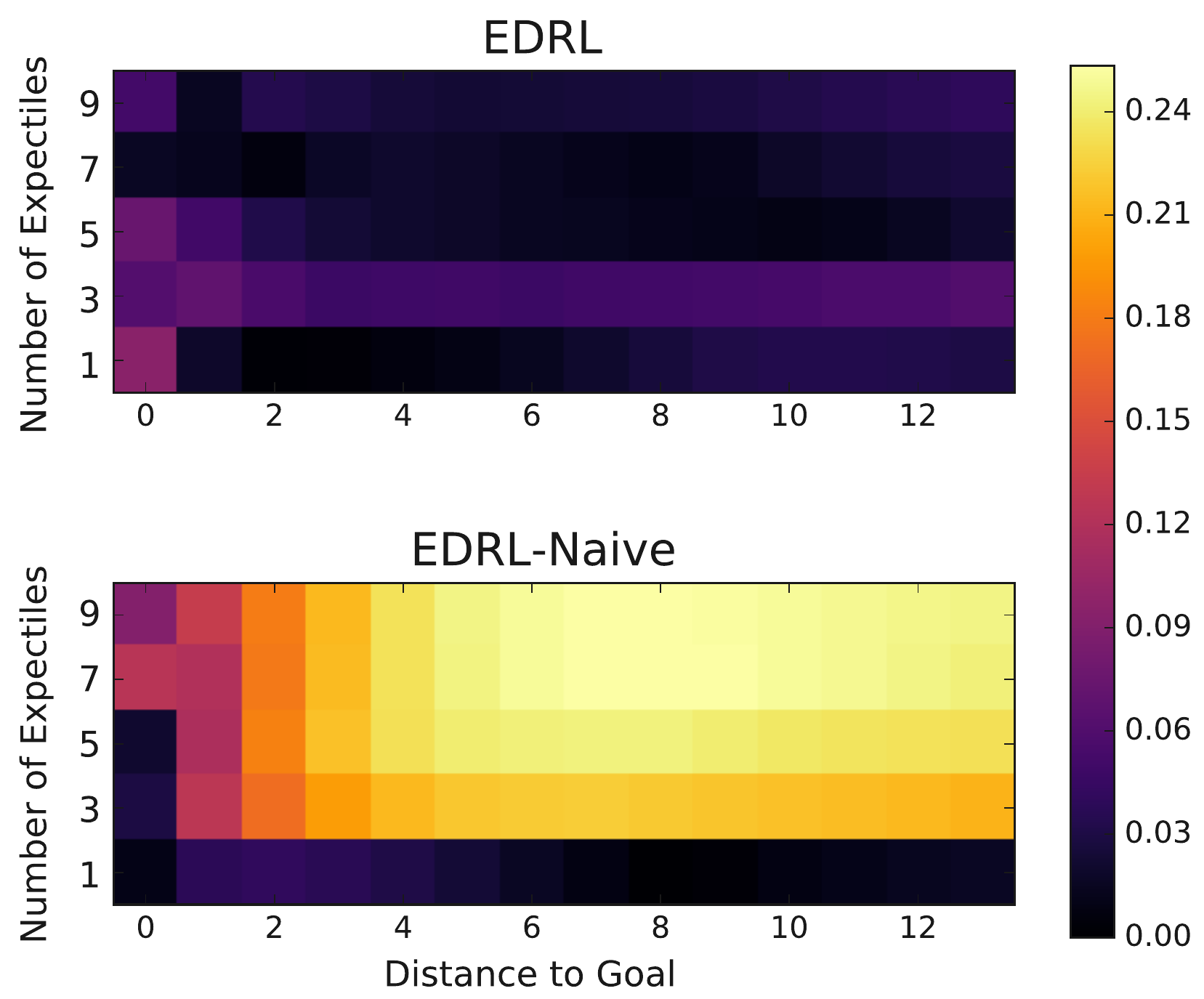}
     }
     \hfill
     \subfloat[Standard Gaussian reward distribution.\label{fig:nchainExpectileErrorGaussian}]{%
       \includegraphics[width=0.38\textwidth]{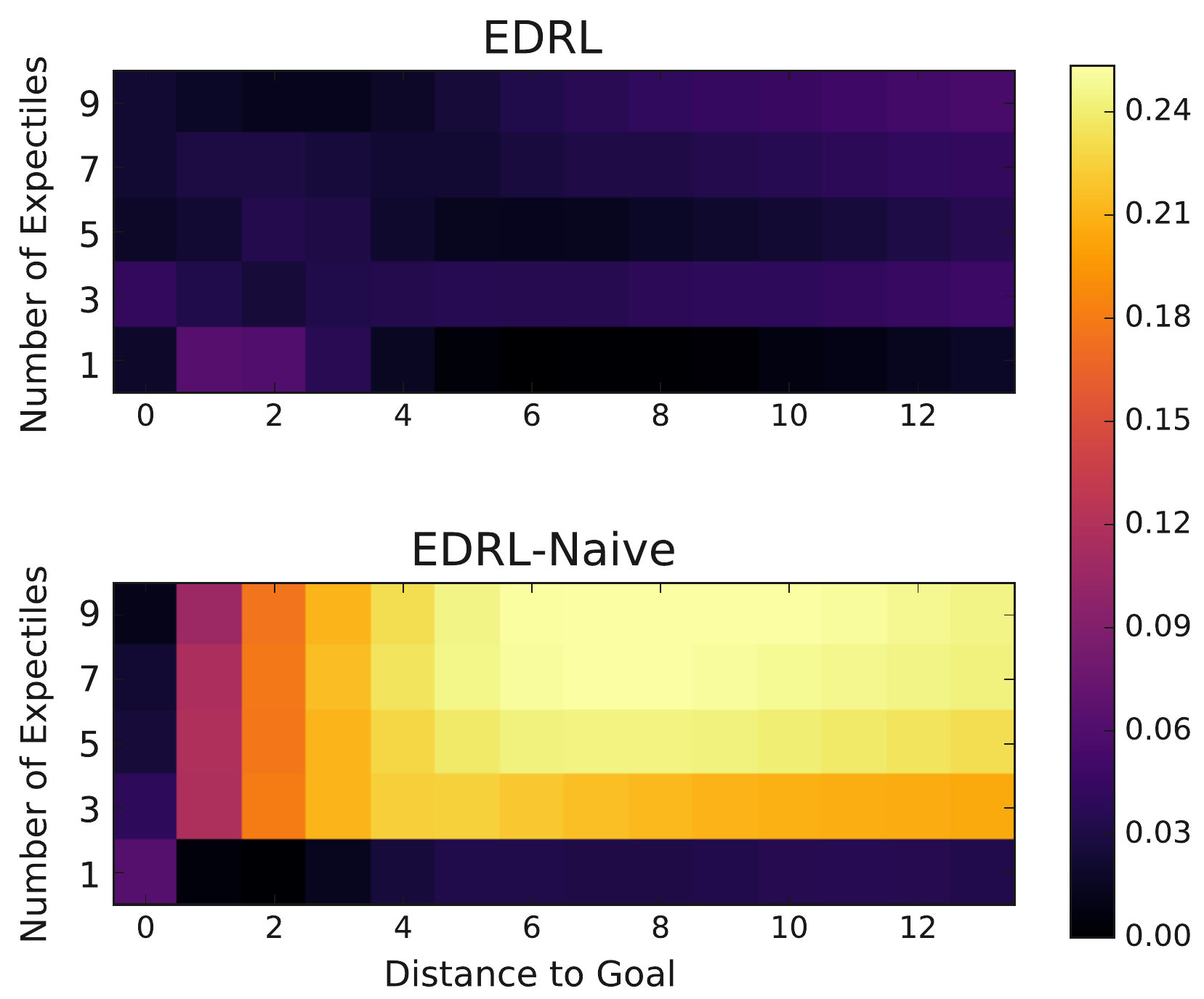}
     }
     \hfill
     \null
     \caption{Expectile estimation error for varying number of expectiles and different chain lengths. Different terminal reward distributions.}
     \label{fig:unifGaussianExpectiles}
\end{figure}

\begin{figure}[h]
    \centering
    \null
    \hfill
     \subfloat[Uniform (-1, 1) reward distribution.\label{fig:nchainW1Uniform}]{%
       \includegraphics[width=0.38\textwidth]{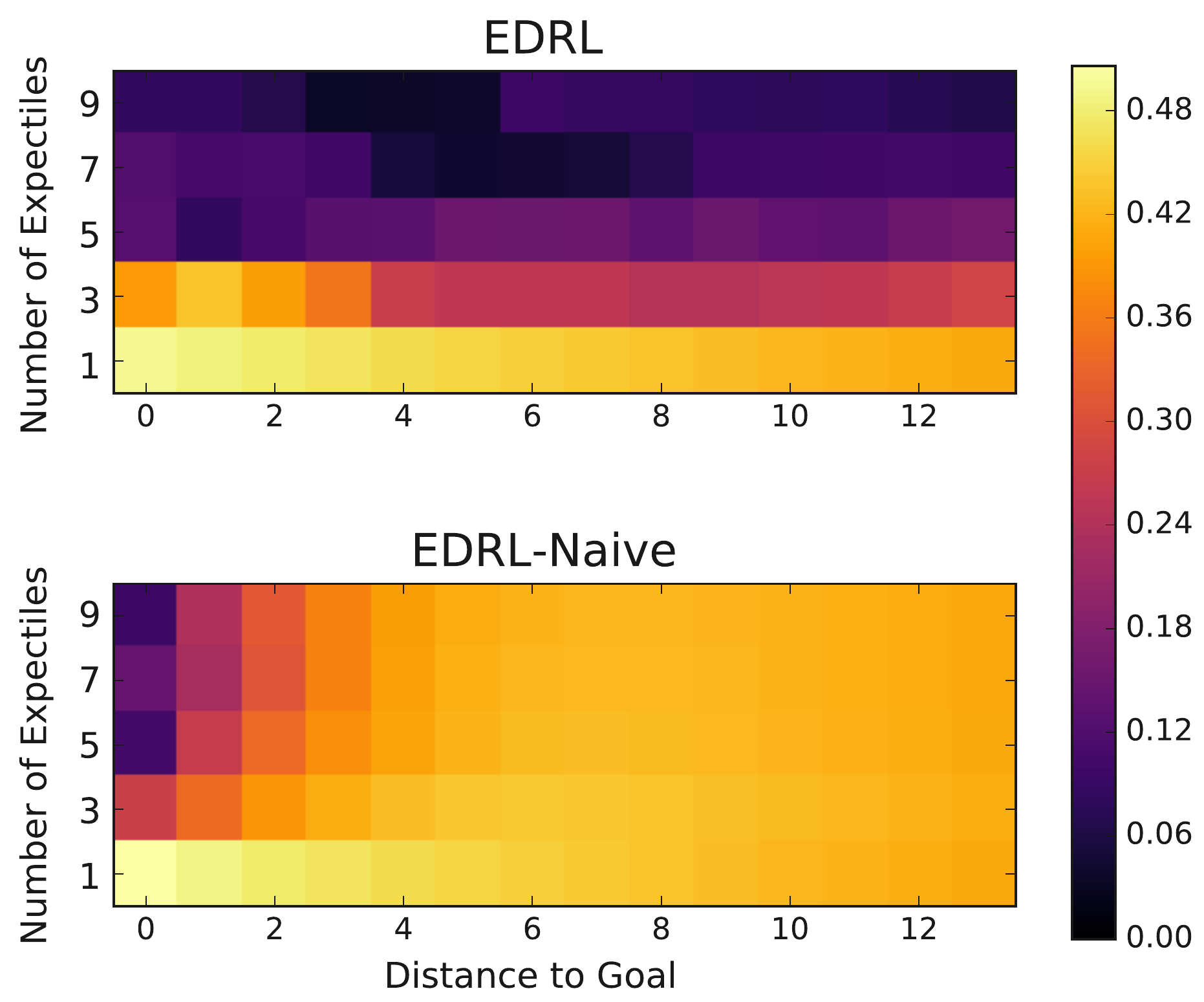}
     }
     \hfill
     \subfloat[Standard Gaussian reward distribution.\label{fig:nchainW1Gaussian}]{%
       \includegraphics[width=0.38\textwidth]{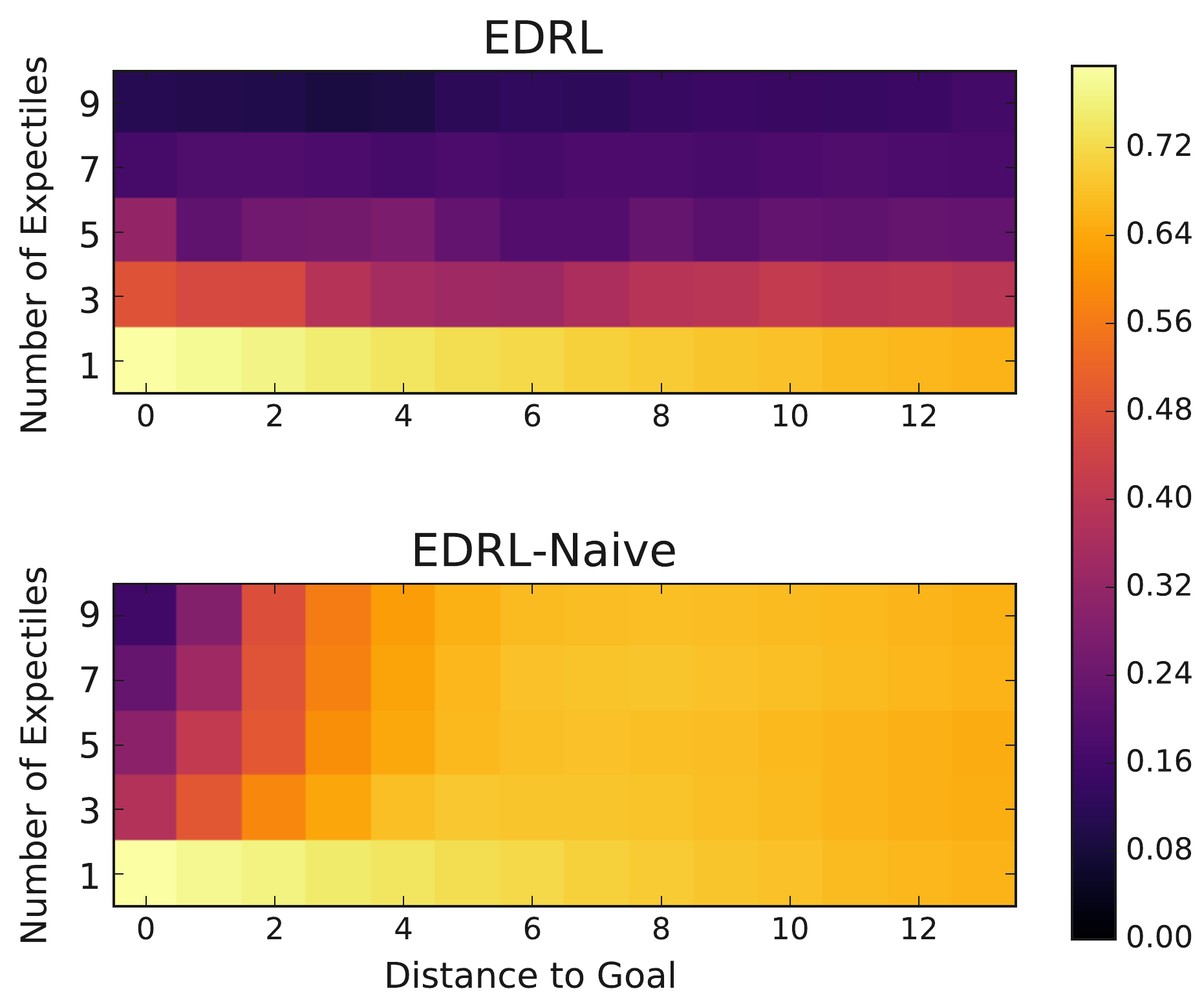}
     }
     \hfill
     \null
     \caption{$1$-Wasserstein distance for varying number of expectiles and different chain lengths. Different terminal reward distributions.}
     \label{fig:unifGaussianWasserstein}
\end{figure}